\documentclass{article} 
\PassOptionsToPackage{table,xcdraw,dvipsnames}{xcolor}
\usepackage{iclr2027_conference,times}

\usepackage{amsmath,amsfonts,bm}

\def\eqref#1{equation~\ref{#1}}

\def\1{\bm{1}}

\def\vzero{{\bm{0}}}

\def\vb{{\bm{b}}}

\def\vu{{\bm{u}}}
\def\vv{{\bm{v}}}
\def\vw{{\bm{w}}}

\def\mW{{\bm{W}}}

\DeclareMathAlphabet{\mathsfit}{\encodingdefault}{\sfdefault}{m}{sl}
\SetMathAlphabet{\mathsfit}{bold}{\encodingdefault}{\sfdefault}{bx}{n}

\def\gF{{\mathcal{F}}}

\def\gH{{\mathcal{H}}}
\def\gI{{\mathcal{I}}}

\def\gP{{\mathcal{P}}}

\newcommand{\E}{\mathbb{E}}
\newcommand{\Ls}{\mathcal{L}}
\newcommand{\R}{\mathbb{R}}

\usepackage{hyperref}
\usepackage{url}

\title{LoCA: Enabling Multi-Modality in Geospatial Location Encoders}

\usepackage{authblk}
\makeatletter
\DeclareRobustCommand{\jointfirst}{\protected@xdef\@thanks{\@thanks\protect\footnotetext[1]{Joint first authorship.}}}
\makeatother

\author{
\textbf{Panagiotis Koromilas}\,$^{* 1,2}$\jointfirst\quad
\textbf{Eleanna Vali}\,$^{* 1,3}$\quad
\textbf{Diogenis Antonopoulos}\,$^{4,5}$\quad
\textbf{Eleni Saka}\,$^{4}$\\
\textbf{Dimitris Christofis}\,$^{1}$\quad
\textbf{Nikolaos Ioannis Bountos}\,$^{6,4}$\quad
\textbf{Konstantin Klemmer}\,$^{7,8}$\\
\textbf{Ioannis Papoutsis}\,$^{3,4,9}$\quad
\textbf{Yannis Panagakis}\,$^{1,3}$\quad
\textbf{Mihalis A. Nicolaou}\,$^{10,2}$
}
\affil{
$^{1}\,$University of Athens\quad
$^{2}\,$The Cyprus Institute\quad
$^{3}\,$Archimedes, Athena Research Center\\
$^{4}\,$National Technical University of Athens\quad
$^{5}\,$Harokopio University\\
$^{6}\,$European Space Agency\quad
$^{7}\,$University College London\quad
$^{8}\,$LGND AI, Inc.\\
$^{9}\,$National Observatory of Athens\quad
$^{10}\,$University of Cyprus
}

\PassOptionsToPackage{numbers, compress}{natbib}
\usepackage{iclr2027_conference}

\usepackage[utf8]{inputenc}
\usepackage[T1]{fontenc}
\usepackage{microtype}
\usepackage{xspace}

\usepackage{amsmath,amsfonts,amssymb,amsthm}
\usepackage{bm}
\usepackage{nicefrac}

\usepackage[table,xcdraw,dvipsnames]{xcolor}
\usepackage{colortbl}
\usepackage{graphicx}
\usepackage{rotating}
\usepackage{tikz}
\usepackage{pgfplots}
\usepackage{subcaption}
\usepgfplotslibrary{groupplots}
\pgfplotsset{compat=1.18}
\usetikzlibrary{arrows.meta, positioning, shapes}

\usepackage{array}
\usepackage{booktabs}
\usepackage{multirow}
\usepackage{multicol}
\usepackage{makecell}
\usepackage{tabularx}
\usepackage{adjustbox}
\usepackage{siunitx}
\usepackage{float}
\usepackage{placeins}
\usepackage{stfloats}
\usepackage{lscape}
\usepackage{afterpage}

\usepackage{enumitem}
\usepackage{algorithmic}
\usepackage{fontawesome5}
\usepackage{pifont}
\usepackage{textcomp}
\usepackage{verbatim}

\usepackage{url}
\usepackage{xr}

\usepackage{hyperref}

\definecolor{cecolor}{RGB}{140,81,10}
\definecolor{ntcecolor}{RGB}{46,139,87}
\definecolor{sntcecolor}{RGB}{255,191,0}
\definecolor{nonlcolor}{RGB}{65,105,225}
\definecolor{cifar10best}{RGB}{0, 120, 0}
\definecolor{cifar100best}{RGB}{200, 100, 0}
\definecolor{baseline}{RGB}{248, 248, 248}
\definecolor{header}{RGB}{250, 250, 250}
\definecolor{bestresult}{HTML}{228B22}
\definecolor{HeaderGray}{RGB}{248,249,250}
\definecolor{AccentGray}{RGB}{242,244,246}
\definecolor{SubtleHighlight}{RGB}{250,251,252}
\definecolor{TextGray}{RGB}{90,90,90}
\definecolor{BorderLight}{RGB}{220,220,220}
\definecolor{MethodGray}{RGB}{110,110,110}
\definecolor{lightgray}{gray}{0.95}
\definecolor{darkgray}{gray}{0.8}
\definecolor{ForestGreen}{RGB}{34, 139, 34}
\definecolor{BrickRed}{RGB}{178, 34, 34}
\definecolor{NavyBlue}{RGB}{0, 100, 200}

\newtheorem{theorem}{Theorem}
\newtheorem{lemma}{Lemma}
\newtheorem{corollary}{Corollary}
\newtheorem{definition}{Definition}

\usepackage{cleveref}

\newcommand{\eg}{\emph{e.g.}\xspace}
\newcommand{\ie}{\emph{i.e.}\xspace}

\crefname{theorem}{Theorem}{Theorems}
\crefname{lemma}{Lemma}{Lemmas}
\crefname{corollary}{Corollary}{Corollaries}
\crefname{definition}{Definition}{Definitions}
\Crefname{theorem}{Theorem}{Theorems}
\Crefname{lemma}{Lemma}{Lemmas}
\Crefname{corollary}{Corollary}{Corollaries}
\Crefname{definition}{Definition}{Definitions}

\usepackage{graphicx}
\usepackage[dvipsnames]{xcolor}
\usepackage{booktabs}
\usepackage{caption}

\newcommand{\locarepo}{\href{https://github.com/eleannavali/LoCA}{\texttt{github.com/eleannavali/LoCA}}}
\providecommand{\E}{\mathbb{E}}
\providecommand{\R}{\mathbb{R}}

\providecommand{\Ls}{\mathcal{L}}
\providecommand{\gF}{\mathcal{F}}
\providecommand{\gH}{\mathcal{H}}
\providecommand{\gI}{\mathcal{I}}
\providecommand{\gP}{\mathcal{P}}

\providecommand{\mW}{\bm{W}}
\providecommand{\vb}{\bm{b}}
\providecommand{\vphi}{\bm{\phi}}
\providecommand{\vu}{\bm{u}}
\providecommand{\vv}{\bm{v}}
\providecommand{\vw}{\bm{w}}
\providecommand{\vzero}{\bm{0}}

\renewcommand{\E}{\mathbb{E}}

\iclrpreprinttrue
\begin{document}

\maketitle

\vspace{-2.2em}
\begin{center}
{\hypersetup{pdfborder={0 0 0}}%
\href{https://github.com/eleannavali/LoCA}{\faGithub\ \texttt{github.com/eleannavali/LoCA}}}
\end{center}
\vspace{0.3em}

\begin{abstract}
 
Location encoders map geographic coordinates to high-dimensional descriptors through contrastive alignment with paired Earth Observation imagery. 
We prove that this construction is structurally limited: 
any Lipschitz location encoder produces an output of intrinsic dimension at most 2, regardless of architecture or bandwidth, which forces the 
contrastive loss to strictly exceed its asymptotic optimum. Empirically,  
joint training also collapses the encoder's embedding
covariance to extremely low rank. 
Although dimension-limited, current location encoders sidestep this spectral collapse by training against a frozen image encoder, which we reinterpret as a covariance regularizer that supplies a high-rank target for the location encoder’s output to fold onto. However, extending this approach to multimodal inputs is fundamentally challenging, as location encoders should simultaneously map a single location to multiple, potentially incompatible, descriptors. 
We identify two requirements a regularizer must satisfy for location encoding in multimodal regimes: \emph{co-alignment} across modalities, since adapters alone cannot align independently pretrained encoders; and \emph{high effective rank}, since the location encoder inherits its target's spectrum. We propose \textbf{LoCA} (\underline{Lo}cation \underline{C}ontrastive \underline{A}lignment), a two-stage pipeline that satisfies both.
\textit{Stage 1} jointly trains Sentinel-1 and Sentinel-2 
encoders with contrastive learning, producing a co-aligned, high-rank multimodal target. \textit{Stage 2} freezes the image encoders and trains a location and a tabular metadata encoder as projection adapters with multi-view contrastive learning. Against five state-of-the-art encoders on 17 tasks ranging from land cover and wildfire forecasting to geophysical, socio-economic, biodiversity and climate applications, LoCA achieves the most task wins, is the only model to outperform every other on a majority of tasks, and wins 68.8\% of all per-task comparisons.
\end{abstract}

\begin{figure}[t]
    \centering
    \includegraphics[width=0.95\linewidth]{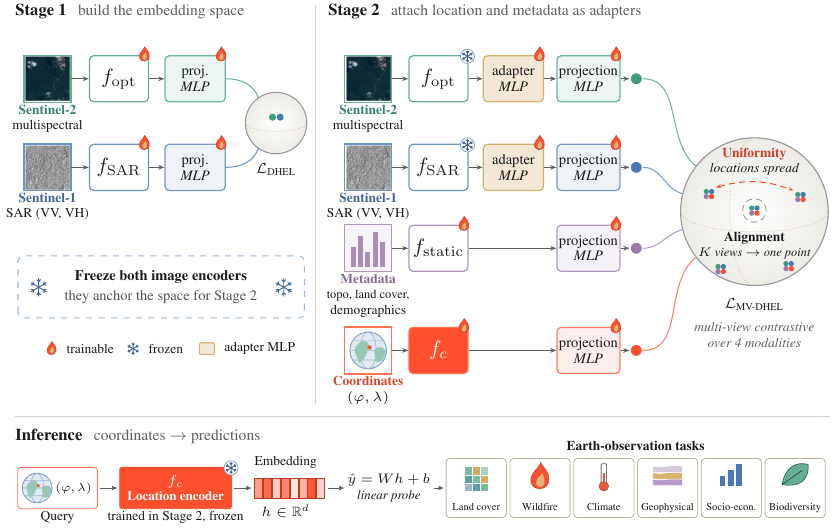}
    \caption{\textbf{LoCA pipeline.} \textbf{\textit{Stage 1}:} jointly train Sentinel-1 SAR and Sentinel-2 optical encoders with DHEL to produce a co-aligned, high-rank multimodal target. \textbf{\textit{Stage 2}:} freeze the image encoders and fit a location  and a metadata encoder as adapters against the four-view system. \textbf{Inference:} a coordinate query passes through the encoder alone to produce an embedding for downstream tasks.}
    \label{fig:loca}
\end{figure}

\section{Introduction}

The rich and heterogeneous information available through Earth Observation (EO) instruments creates massive opportunities for understanding the complex processes governing our planet. Each EO sensor captures distinct aspects of the Earth's surface, from spectral reflectance to atmospheric dynamics, yet no single modality offers a complete representation. Effectively distilling and fusing these diverse sources holds the potential for a universal descriptor of any location on the planet. Such a \textit{location encoder} could propel research across applications with inherent multimodal dependencies, including wildfire forecasting \citep{prapas2023televit}, species distribution modeling \citep{sastry2024birdsat}, and land cover classification \citep{sumbul2021bigearthnet}.

A large body of geospatial applications rely on the persistent geophysical and semantic properties of a location. Recent work on EO foundation models \citep{tseng2025galileo,jakubik2025terramind,brown2025alphaearth,wang2025towards} aims at encoding such diverse information applicable to arbitrary downstream tasks. However, many such models are trained as functions of imagery: an observation image is required at inference, while the same location admits many vectors depending on which image is supplied. Coordinate-indexed embedding data products published from these models \citep{brown2025alphaearth,wang2025towards} are snapshots of the imagery at release time, not functions on coordinates, and cannot return a vector outside the published grid or the release-time observation window. The desired vector that describes the inherent characteristics of a location needs a function defined on all coordinates, easily queryable from latitude and longitude alone. 

Contrastive learning is the natural fit for building such a function. The InfoNCE objective \citep{oord2018representation,chen2020simple} reduces to categorical cross-entropy over a positive and $K-1$ negatives, the same loss shape as a linear classifier on a frozen representation, and produces globally separable semantic embeddings \citep{wang2020understanding}. Contrastive location encoders, such as SatCLIP \citep{klemmer2025satclip}, GeoCLIP \citep{cepeda2023geoclip} and CSP \citep{pmlr-v202-mai23a}, align location embeddings to frozen imagery embeddings with a two-view InfoNCE objective.  
However, we show that \textit{a location encoder is structurally incompatible with the geometry of contrastive learning}. 

First, we note that a Lipschitz map cannot increase intrinsic dimension, so any location encoder ($f_c$), including the spherical-harmonic and Siren architectures used by SatCLIP, sends $S^k$ onto a set of intrinsic dimension at most $k$. With $k=2$ for geographic coordinates and embedding dimension $d{-}1=255$, the location embedding lies on a $2$-dimensional surface inside a $255$-dimensional sphere, regardless of bandwidth or parameterization. Second, we prove that this is the wrong substrate for contrastive learning: the asymptotic optimum of the InfoNCE objective is the uniform distribution on $S^{d-1}$ \citep{wang2020understanding}, which such a surface cannot represent. Consequently, the contrastive loss strictly exceeds its unconstrained minimum at every aligned Lipschitz pair (\Cref{thm:lipschitz}), uniformly over SatCLIP-type encoder families (\Cref{cor:satclip}). Empirically, we further observe that joint training exhibits a spectral collapse: alignment drags the image encoder onto the same surface, and the empirical covariance of the \textbf{location encoder collapses to low rank}. A $255$-dimensional embedding sphere ends up used along only a few linearly independent directions, severely limiting downstream concept separability.

Despite these geometric limitations, existing location encoders remain effective. We argue this is because, in practice, they freeze a pretrained unimodal encoder and thereby avoid the empirical spectral collapse. SatCLIP \citep{klemmer2025satclip} and GeoCLIP \citep{cepeda2023geoclip} keep a pretrained image encoder fixed and only fit the location side, which changes the problem from \emph{learning a representation jointly with another encoder} to \emph{fitting an inverse map from coordinates to a fixed point cloud on the embedding sphere}. The geometric image of $f_c$ remains $2$-dimensional, but its empirical sample and spectrum now track those of the pretrained target. The frozen image encoder thus acts as a covariance regularizer, turning $f_c$ into an adapter onto a high-rank cloud the pretrained model already produced.

However, \textit{this paradigm does not extend to multiple sensors}. Two independently pretrained satellite image encoders, a Sentinel-1 SAR model and a Sentinel-2 optical model, \textbf{are not co-aligned}~\citep{zavras2025mind}: nothing in their separate objectives places paired inputs at the same point, so no single location function can fit both. Trainable MLP adapters absorb these per-modality transforms but cannot manufacture co-alignment that isn't already there.
Multimodal location encoders therefore need a target that is co-aligned by construction. Motivated by the above limitations, we introduce \textbf{LoCA} (\underline{Lo}cation \underline{C}ontrastive \underline{A}lignment), a two-stage pipeline determined by two key requirements for the targets: \emph{co-alignment} across modalities, so that the location function fits one constraint per coordinate rather than one per modality; and \emph{high covariance rank}, so that the spectrum the location encoder inherits is not itself collapsed. \textit{Stage 1} enforces co-alignment by jointly training a Sentinel-1 and a Sentinel-2 satellite image encoders with DHEL\citep{koromilas2024bridging}, the contrastive objective that attains the highest effective rank. \textit{Stage 2} freezes both image encoders, attaches a small projection adapter to each, and trains a location encoder $f_c$ and a metadata encoder $f_m$ (MLP on tabular data) against the four-encoder pipeline with MV-DHEL \citep{koromilas2025principled}, the multi-view extension of DHEL. In this way, LoCA produces multimodal location embeddings of high variability across the 2 available intrinsic dimensions. We extensively benchmark LoCA comparing it against state-of-the-art baselines on a diverse benchmark suite covering all relevant semantic factors of variability inherent in the EO domain as defined in \cite{kondylatos2025generalization}. Our contributions can be summarized as follows:

\begin{enumerate}[label=\textbf{C\arabic*:},leftmargin=*]
    \item We \textbf{prove} that location encoders constrain contrastive learning away from its optimum, and empirically show that \textbf{image-location contrastive training from scratch collapses}.
    
    \item We recast existing contrastive location encoders as adapters over pretrained image representations and argue that extending this formulation to multiple modalities requires (i) \textbf{cross-modal co-alignment} and (ii) \textbf{high effective rank} of the target space.
    
    \item  We \textbf{extend location encoders to utilize multimodal satellite representations and  metadata} by introducing \textbf{LoCA}, a training pipeline achieving co-aligned, high-rank features by first jointly training Sentinel-1 and Sentinel-2 encoders and then adapting location and metadata encoders.
    
    \item LoCA achieves the \textbf{best performance on more tasks} than state-of-the-art competitors and \textbf{beats each examined location encoders} on head-to-head comparisons.
\end{enumerate}

Overall, our work provides a \textbf{principled framework for location encoders}: it identifies their failure modes, derives the requirements for extending them to multiple modalities, and uses them to enable multimodality through LoCA, laying the groundwork for future advances in location encoding.

\section{Related Work}
\label{sec:related}
\noindent\textbf{Geospatial Foundation Models (GFMs) \& Embeddings.}
Several GFMs use Masked Image Modelling (MIM) with various inductive biases: multi-scale or scale-aware encodings \cite{reed2023scale, tang2023cross}, spectral-aware masking \cite{irvin2023usat, zhu2023spectralmae}, and sensor-agnostic backbones that absorb heterogeneous EO inputs \cite{tseng2023lightweight, bountos2025fomo, xiong2024neural, wang2025towards, herzog2025olmoearth, forgaard2026thor}. A second family blends MIM with contrastive or reconstruction objectives across satellite modalities \cite{tseng2025galileo, guo2024skysense, fuller2023croma, astruc2024omnisat}, and a third uses generative or diffusion-based formulations \cite{jakubik2025terramind, khanna2023diffusionsat, jia2025can}. Building on this momentum, recent works release precomputed embedding products at fixed grids: AlphaEarth~\cite{brown2025alphaearth} fuses optical, radar, LiDAR, and climate reanalysis into $64$-dimensional annual embeddings at $10$\,m, Copernicus-FM~\cite{wang2025towards} provides $0.25^{\circ}$ multi-sensor averages, and TESSERA~\cite{feng2026tessera} releases annual Barlow-Twins-pretrained embeddings with \textit{int8} quantization. These products are observation-based rather than location-based: each cell is a snapshot of imagery tied to a specific acquisition window, not a function of coordinates. Recent work positions embeddings generated via vision models as one instantiation of an ``Earth embedding'', an AI-native representation of a given point in geographic space (and time) \cite{klemmer2026earth}.

\noindent\textbf{Location Encoders.}
A second branch of models skips the image encoding step and produces embeddings directly as a function of coordinates. Early designs train MLPs to predict object presence from coordinates~\cite{MacAodha2019PresenceOnlyGP}, followed by geodesic-preserving multi-scale spherical encodings~\cite{mai2023sphere2vec} and the spherical-harmonic / sinusoidal-MLP architecture~\cite{russwurm2023geographic} that has since become standard. Contrastive learning is now the dominant regime for self-supervised training: GeoCLIP~\cite{cepeda2023geoclip} aligns coordinates with CLIP image embeddings, and SatCLIP~\cite{klemmer2025satclip} aligns them with embeddings stemming from frozen Sentinel-2 encoders to produce general-purpose location embeddings. Multimodal extensions either add new contrastive losses per modality~\cite{liu2025gair, shatwell2025gt} or align all pairwise combinations of image, location, and time through separate InfoNCE terms~\cite{shatwell2026tiger}. These designs share two structural costs: heterogeneous loss terms produce conflicting gradients, and the number of terms grows quadratically with the number of modalities. We address both by generalizing the SatCLIP framework to arbitrary modalities through a single multi-view objective.

\noindent\textbf{Multi-View Contrastive Learning.}
Multi-view CL has been studied through aggregations of pairwise losses across heads, scales, classes, or feature levels \citep{wang2024adaptive, li2023global, shah2023multi, xu2022multi}, by treating each view combination as an independent task \citep{tian2020contrastive}, through alternative aggregations \citep{shidani2024polyview}, and via optimal transport over similarity tensors\cite{piran2024contrasting}. MV-DHEL \citep{koromilas2025principled} is the multi-view objective whose optimum scales with the number of views without dimensional collapse, making it suitable to train our multimodal location encoders.

\section{Revisiting the inner workings of Location Encoders}
In this section we explore the modeling limitations in Location Encoders. In \cref{sec:prelims,sec:thm} we formulate the task of Location Encoding and provide theoretical bounds on the optima of contrastive learning with location encoders and the dimensionality of the produced representations. In \cref{sec:label-assigner} we analyse the inner workings of existing location encoders, and reformulate them through the lens of a \textit{coordinate-indexed lookup table}.

\label{sec:optimal-geometry}

\subsection{Setup and Notation}
\label{sec:prelims}

Let the location domain be the $k$-sphere $S^k \subset \R^{k+1}$ with an atomless probability measure $\mu$ (for SatCLIP, $k = 2$ and $\mu$ is the geographic distribution of training locations on Earth). The image domain $\gI$ is an arbitrary measurable space. Pairs $(c, I)$ are drawn from a joint distribution $p$ with each image tied to a unique location, $p(c \mid I) = \delta_{c(I)}$. Two encoders $f_c : S^k \to S^{d-1}$ and $f_I : \gI \to S^{d-1}$ share the embedding sphere of dimension $d - 1$, with $d - 1 > k$ (for SatCLIP, $d - 1 = 255$). They minimize the population NT-Xent loss \citep{wang2020understanding,koromilas2024bridging}
\begin{equation}
\begin{split}
\Ls_\tau(f_c, f_I) = -\E_{(c,I) \sim p}\!\!\left[\frac{f_c(c)^\top f_I(I)}{\tau}\right] &+ \tfrac{1}{2}\, \E_c \log \E_{I'}\!\!\left[e^{f_c(c)^\top f_I(I')/\tau}\right] \\
&+ \tfrac{1}{2}\, \E_I \log \E_{c'}\!\!\left[e^{f_c(c')^\top f_I(I)/\tau}\right],
\end{split}
\label{eq:popnce}
\end{equation}
the population (infinite-batch) limit of the symmetric CLIP loss \citep{pmlr-v139-radford21a}. The first term rewards same-pair embeddings being close, while the next two push different-pair embeddings apart.

We use the \emph{pushforward measure} $f_\#\mu$ on $S^{d-1}$, the distribution of $f(X)$ when $X \sim \mu$, so $f_{c\#}\mu$ is the distribution of location embeddings produced by $f_c$. We will use the \emph{Hausdorff dimension} $\dim_H E$, which is related to an intrinsic-dimension notion on smooth manifolds and extends to irregular sets (definition in \Cref{app:prelims}). We write $p$-a.e.\ for ``$p$-almost every,'' meaning except on a set of $p$-measure zero. Let $\Ls_\tau^{\mathrm{free}}$ denote the infimum of $\Ls_\tau$ over all measurable encoder pairs. Among aligned pairs (satisfying $f_c(c) = f_I(I)$ for $p$-a.e.\ $(c, I)$), this infimum is attained when additionally $f_{c\#}\mu = \sigma$, the uniform measure on $S^{d-1}$ \citep{wang2020understanding}. We denote $V^*$ for the value at this aligned-uniform optimum: $V^* \ge \Ls_\tau^{\mathrm{free}}$ by definition, with the reverse inequality not generally true because of the modality gap \citep{liang2022mind,cai2026geometric}. We prove a strict gap to $V^*$, which suffices for a strict gap to $\Ls_\tau^{\mathrm{free}}$.

\subsection{Location encoders constrain contrastive learning away from its optimum}
\label{sec:thm}

\begin{theorem}[Dimensional gap for Lipschitz encoders]
\label{thm:lipschitz}
Let $k \geq 1$, $d > 2$ with $k < d - 1$. Let $f_c : S^k \to S^{d-1}$ be Lipschitz and $f_I : \gI \to S^{d-1}$ measurable. Set $\nu := f_{c\#}\mu$. Then
\begin{enumerate}[topsep=2pt,itemsep=2pt,leftmargin=2.2em]
\item[(i)] $\dim_H f_c(S^k) \le k$. The location encoder maps $S^k$ onto a set of intrinsic dimension at most $k$, regardless of how $f_c$ is parameterized or how high the embedding dimension $d$ is.
\item[(ii)] $\nu \perp \sigma$, in the sense that $\nu$ and $\sigma$ live on disjoint Borel subsets of $S^{d-1}$. The location-embedding distribution and the uniform target on the sphere have zero overlap.
\item[(iii)] If $(f_c, f_I)$ is aligned, $\Ls_\tau(f_c, f_I) > \Ls_\tau^{\mathrm{free}}$. The contrastive loss strictly exceeds its unconstrained minimum at every aligned Lipschitz pair.
\end{enumerate}
\end{theorem}

The corresponding proof is provided in \Cref{app:thm-proof}. From a high level, we prove that no aligned pair with a Lipschitz location encoder attains the optimum of contrastive learning. Part (i) states that a Lipschitz map \textbf{cannot increase intrinsic dimension}. Part (ii) shows that the location embeddings and the uniform target (the optimum of contrastive learning) live on disjoint pieces of the sphere with zero overlap. The consequence is that (iii) under alignment, the strict mismatch between $\nu$ and $\sigma$ forces the loss strictly above its unconstrained minimum.

The above findings have direct consequences for existing location encoders. SatCLIP composes a spherical-harmonic feature map $\mathrm{SH}_L : S^2 \to \R^{(L+1)^2}$ of degree $L$, a Siren MLP, and a normalization onto $S^{d-1}$ \citep{klemmer2025satclip,sitzmann2020implicit}. We write $\gF_L^{\mathrm{SC}}$ for the resulting bandwidth-$L$ family under a mild non-degeneracy condition (definitions in \Cref{app:uniflip}). Each piece of the composition is Lipschitz with a constant independent of the trainable parameters, so every $f_c \in \gF_L^{\mathrm{SC}}$ is Lipschitz with a constant uniform across the family (\Cref{lem:uniflip}), and \Cref{thm:lipschitz} applies pointwise.

\begin{corollary}[SatCLIP strict positivity over aligned pairs]
\label{cor:satclip}
Fix $d \geq 4$. For every finite $L \in \mathbb{N}$ and every admissible parameter set, the contrastive loss of any aligned SatCLIP pair is strictly larger than the unconstrained free optimum: even when minimizing jointly over the entire family $\gF_L^{\mathrm{SC}}$ and over all aligned image encoders, the infimum stays strictly above $\Ls_\tau^{\mathrm{free}}$.
\end{corollary}

These results extend from a single Lipschitz encoder to the infimum over the entire bandwidth-$L$ class, so no parameter setting reaches the free optimum. See \Cref{app:cor-proof} for the proof. 

\begin{figure}[t]
  \centering
  \begin{minipage}[t]{0.54\textwidth}
    \vspace{0pt}
    \centering
    \includegraphics[width=0.95\linewidth]{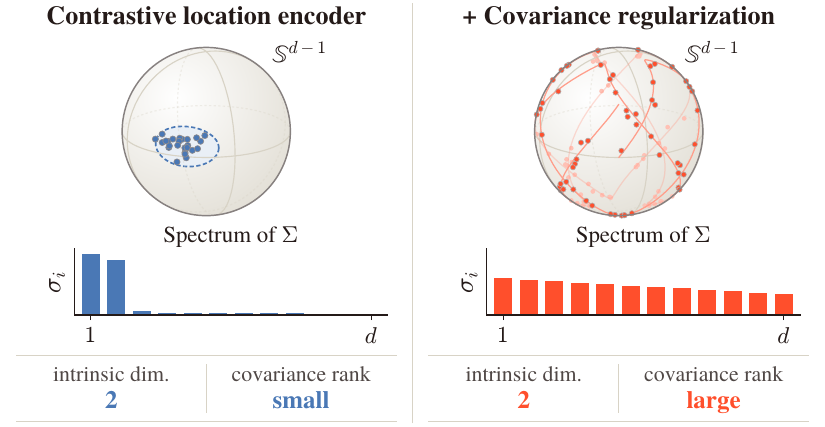}
    \caption{\textbf{Intrinsic dimension is bounded; the covariance is not.}
      Plain contrastive learning collapses the embedding to a 2-d patch with low variability; covariance regularization makes the 2-d manifold wind densely through $S^{d-1}$, raising the covariance rank without changing the intrinsic dimension.}
    \label{fig:id_vs_cov}
  \end{minipage}\hfill
  \begin{minipage}[t]{0.42\textwidth}
    \vspace{0pt}
    \captionof{table}{Effective rank \cite{roy2007effective} of the location-embedding marginal and its ratio to the total representation dimensionality $D$. Higher = more directions used. ``Anchor'' indicates the pretrained target whose spectrum the encoder starts from. \textcolor{NavyBlue}{Blue} indicates training on Sentinel-2, \textcolor{Plum}{Purple} on RGB and \textcolor{PineGreen}{Green} on both Sentinel-1\&2.}
    \label{tab:effrank}
    \vspace{0.6em}
    \centering
    \footnotesize
    \setlength{\tabcolsep}{2.5pt}
    \renewcommand{\arraystretch}{1.15}
    \begin{tabular}{@{}llcc@{}}
      \toprule
      \textbf{Method}             & \textbf{Anchor} & \textbf{Erank} & \textbf{\%\,of $D$} \\
      \midrule
      SatCLIP$^{\dagger}$         & No           & 8   & 3.1  \\
      SatCLIP                     & \textcolor{NavyBlue}{MoCo R50}     & 89  & 34.9 \\
      Copernicus-FM               & \textcolor{Plum}{DINOv2}       & 158 & 20.6 \\
      \textbf{LoCA (ours)}        & \textcolor{PineGreen}{Stage 1} & 158 & 62   \\
      \bottomrule
    \end{tabular}
    \par\vspace{0.3em}
    {\scriptsize $^{\dagger}$image encoder trained from scratch (no MoCo init).}
  \end{minipage}
\end{figure}
\subsection{Geometry, variability collapse, and the role of regularization}
\label{sec:label-assigner}

\Cref{thm:lipschitz}(i) states that the image of a location encoder has intrinsic dimension at most $k$ inside $S^{d-1}$, regardless of architecture or bandwidth. This is intuitively expected: standard encoders compress high-dimensional inputs into a lower-dimensional code; a location encoder does the opposite, decoding a low-dimensional coordinate into a higher-dimensional vector. As a result, \textit{location encoders constitute the wrong substrate for contrastive learning} as they can never attain the point configurations that the training objective optimizes for. \Cref{thm:lipschitz} constrains the geometry of the embedding, not its spectrum: a $k$-dimensional surface can still have full-rank covariance (\Cref{fig:id_vs_cov}). Under \Cref{eq:popnce}, alignment forces the essential range of $f_I$ to lie in $f_c(S^k)$ (\Cref{app:thm-proof}), so both encoders inherit the location side's intrinsic dimension. Empirically, joint training also settles on a surface of low covariance rank (\Cref{tab:effrank}): a $255$d embedding space is used along only a small number of linearly independent directions.

A natural fix is to break the symmetry of joint training by freezing one side, creating a setup in which contrastive objectives can work. Freezing $f_I$ regularizes $f_c$: instead of co-collapsing with another encoder, $f_c$ is fitted against a fixed, externally produced point cloud $\{y_i := f_I(I_i)\}_{i \le N}$ on $S^{d-1}$. The geometric image $f_c(S^k)$ remains $k$-dimensional by \Cref{thm:lipschitz}, but the empirical sample $\{f_c(c_i)\}_{i \le N}$ now tracks $\{y_i\}$, and its spectrum tracks the spectrum of the target cloud. The function thus memorizes high-rank labels by oscillating along its $k$-dimensional surface. Geometry and spectrum decouple (\Cref{fig:id_vs_cov}) since  the manifold stays low-dimensional while the empirical covariance reaches whatever rank the target supplies.

This is \textit{why} existing contrastive location encoders work. SatCLIP \citep{klemmer2025satclip} and GeoCLIP \citep{cepeda2023geoclip} are not performing contrastive learning in the usual sense of jointly learning a representation with another encoder; they instead use the contrastive objective as a \textit{proxy for fitting a coordinate-indexed lookup table} whose values come from a pretrained image encoder acting as a regularizer that provides $f_c$ a high-rank target to fold onto. This adjustment to the image modality is also effective in other domains through the Locked-image Tuning \citep{9878889}.

\section{Adapting Location Encoders to Multimodal Embeddings}
\label{sec:method}

In \Cref{sec:label-assigner} the location encoder is reframed as an adapter that indexes a pretrained label cloud. We now ask: {{\emph{what must a label cloud satisfy for this view to extend to multiple sensors?}} We identify two requirements in \Cref{sec:multimodal-fails} and \Cref{sec:regularization} 
and introduce LoCA that satisfies both in \Cref{sec:loca}. 

\subsection{A location encoder cannot align to two targets}
\label{sec:multimodal-fails}
Extending the \textit{coordinate-indexed lookup table} paradigm to multiple modalities is far from trivial, since the targets it supplies must be mutually consistent across modalities. Formally, suppose we freeze a Sentinel-1 encoder $f_{I_1}$ and a Sentinel-2 encoder $f_{I_2}$, both pretrained independently, and demand a single $f_c$ to align with both. Alignment requires $f_c(c_i) = f_{I_1}(I_i^{S_1}) =: y_i^{S_1}$ and $f_c(c_i) = f_{I_2}(I_i^{S_2}) =: y_i^{S_2}$ on every paired sample, which forces $y_i^{S_1} = y_i^{S_2}$. Independently trained encoders do not satisfy this: their embedding spaces are determined up to an arbitrary orthogonal transformation by their respective objectives and initializations, and there is no mechanism by which paired inputs would land at the same point. The location encoder is presented with two incompatible label sets and cannot fit either. Inserting trainable projection adapters between $f_c$ and each frozen target relaxes this rigidity by absorbing per-modality orthogonal transformations. However, adapters cannot manufacture co-alignment that is absent from the joint distribution of the targets themselves: no choice of adapters can pull $y_i^{S_1}$ and $y_i^{S_2}$ to the same point unless the underlying $f_{I_1}, f_{I_2}$ already agree up to such a transformation. This is a structural limitation and extends to any setup combining $M$ not co-trained encoders. Adapting to multiple modalities therefore requires a target that is \textbf{\textit{co-aligned by construction}}: paired-sample labels for distinct modalities must coincide on $S^{d-1}$, so that $f_c$ is asked to satisfy one constraint per coordinate rather than $M$. 

\subsection{Regularizer embeddings should have high rank}
\label{sec:regularization}

Through the regularizer view of \Cref{sec:label-assigner} the number of linearly independent directions in the representation space of $f_c$ is bounded to that of the target. Alignment forces $f_c(c_i) \approx y_i$, so a low-rank target propagates its collapse to the location encoder. This rules out a natural choice of target \ie, multimodal GFMs, primarily relying on MIM objectives (\Cref{sec:related}). MIM objectives concentrate variance on a small number of directions, a pathology documented for self-supervised methods broadly \citep{jing2022understanding} and for GFMs specifically \citep{hackel2026slimming,robinson2026compressing,rao2025measuring}. 
Using such a model as a pretrained anchor would serve $f_c$ a low-rank label cloud and reproduce the very collapse the adapter is supposed to avoid. We call this the \textbf{\textit{rank requirement}}. This can be met in two ways: a) guide the representation space with external, high-rank anchors \eg Copernicus-FM distills from Dinov2, or b) utilize optimization processes that enforce a high-rank. The first introduces unnecessary axes of variation potentially reflecting exogenous features (\eg from RGB natural imagery) which may limit its utility to domain-relevant applications (see \cref{sec:benchmark}), while the latter creates a label cloud for the target modalities themselves. 

\subsection{LoCA: Location Contrastive Alignment}
\label{sec:loca}
 We introduce LoCA (\underline{Lo}cation \underline{C}ontrastive \underline{A}lignment), a two-stage pipeline that addresses all aforementioned limitations. The full pipeline is illustrated in \Cref{fig:loca} and detailed below.

\textbf{\textit{Stage 1}: Jointly Contrastive Multimodal Pretraining.} The first stage involves training information-rich
Sentinel-1 and Sentinel-2 encoders jointly with the DHEL loss, which enforces the representation space to utilize a big fraction of the available dimensionality \citep{koromilas2024bridging}. Furthermore, joint contrastive training enforces co-alignment as a property of the loss optimum, since paired inputs are pulled to the same embedding by construction.
These two properties result in a co-aligned, high-rank georeferenced label cloud $\{y_i^{\mathrm{multi}}\}$, which can be effectively used as a target.

\textbf{\textit{Stage 2}: Multi-View Contrastive Adapter Fitting.} In this stage we freeze both image encoders, letting the location encoder adapt to a fixed multimodal embedding space. We attach a trainable MLP projection to each uni-modal encoder, and introduce a location encoder $f_c$  and a metadata encoder $f_m$ (a standard MLP that processes all 14 variables as a single feature vector). Notably, \Cref{thm:lipschitz} applies to $f_m$ as well with $k=14$, restricting its usage as a \textit{Stage 1} anchor. We train this four-encoder system with the MV-DHEL \citep{koromilas2025principled} loss. MV-DHEL is a multi-view extension of DHEL whose optimum scales with the number of views without dimensional collapse. Because both image encoders are frozen, their embeddings remain co-aligned and the target space retains its high effective rank. As a result, the empirical spectrum of $f_c$ at training reaches the projection-head dimension, even though $f_c(S^2)$ is still a $2$-dimensional manifold by \Cref{thm:lipschitz}.

\begin{table*}[]
  \caption{\textbf{Downstream evaluation. }\underline{Mean over 10 seeds} (3 for SatBird); metric $\times\,100$, higher is better. {\color{ForestGreen}\textbf{Green}}: best across both probes; {\color{RoyalBlue}\textbf{blue}}: best within probe (MLP, Linear) when not best overall. Environmental tasks predict relative ordering of climate variables. Bottom two rows: head-to-head comparisons (\Cref{tab:head-to-head}). N/A: embeddings not retrievable (\Cref{app:Evaluation}).}
  \label{tab:benchmark-results-main}
  \centering
  \resizebox{\textwidth}{!}{%
  \begin{tabular}{lccccccccccccc}
    \toprule
    & \multicolumn{4}{c}{\textbf{Observation-based FMs}} & \multicolumn{8}{c}{\textbf{Location encoders}} & \\
    \cmidrule(lr){2-5} \cmidrule(lr){6-13}
    \textbf{Task $\downarrow$ Model $\rightarrow$}
      & \multicolumn{2}{c}{\textbf{AlphaEarth}}
      & \multicolumn{2}{c}{\begin{tabular}{@{}c@{}}\textbf{Copernicus}\\[-0.3em]{\scriptsize FM}\end{tabular}}
      & \multicolumn{2}{c}{\textbf{SatCLIP}}
      & \multicolumn{2}{c}{\textbf{GeoCLIP}}
      & \multicolumn{2}{c}{\textbf{GT-Loc}}
      & \multicolumn{2}{c}{\begin{tabular}{@{}c@{}}\textbf{LoCA}\\[-0.3em]{\scriptsize (Ours)}\end{tabular}} & \multirow{2}{*}{\textbf{Metric}} \\
    \cmidrule(lr){2-3} \cmidrule(lr){4-5} \cmidrule(lr){6-7} \cmidrule(lr){8-9} \cmidrule(lr){10-11} \cmidrule(lr){12-13}
      & MLP & Linear & MLP & Linear & MLP & Linear & MLP & Linear & MLP & Linear & MLP & Linear & \\
    \midrule
    \multicolumn{14}{l}{\textit{\textbf{Biodiversity \& Ecology}}} \\
    \addlinespace[1pt]
    Countries
      & N/A & N/A
      & 80.9 & 88.3
      & 95.5 & 96.5
      & 91.5 & 79.8
      & 94.7 & 94.4
      & {\color{RoyalBlue}\textbf{96.4}} & {\color{ForestGreen}\textbf{97.0}}
      & \textit{F1$_{w}$} \\
    Biome
      & N/A & N/A
      & 84.0 & 82.4
      & 89.8 & 86.6
      & 86.9 & 75.6
      & 88.1 & 81.9
      & {\color{ForestGreen}\textbf{91.0}} & {\color{RoyalBlue}\textbf{87.0}}
      & \textit{F1$_{w}$} \\
    Ecoregions
      & N/A & N/A
      & 59.9 & 75.1
      & 79.7 & 84.4
      & 60.6 & 42.6
      & 69.6 & 75.9
      & {\color{RoyalBlue}\textbf{80.2}} & {\color{ForestGreen}\textbf{85.0}}
      & \textit{F1$_{w}$} \\
    iNaturalist
      & 49.1 & 59.4
      & 40.4 & 61.8
      & {\color{RoyalBlue}\textbf{56.4}} & 69.8
      & 44.8 & 58.4
      & 49.3 & 59.3
      & 55.0 & {\color{ForestGreen}\textbf{71.4}}
      & \textit{Acc.} \\
    SatBird (USA summer)
      & {\color{ForestGreen}\textbf{67.5}} & 56.5
      & 59.6 & 57.9
      & 60.7 & 58.2
      & 63.3 & 47.8
      & 63.3 & 52.7
      & 60.8 & {\color{RoyalBlue}\textbf{61.1}}
      & \textit{top-$k$ Acc.} \\
    \midrule
    \multicolumn{14}{l}{\textit{\textbf{Land cover}}} \\
    \addlinespace[1pt]
    reBEN
      & 42.8 & 43.3
      & 57.4 & {\color{RoyalBlue}\textbf{55.2}}
      & 57.3 & 54.9
      & 56.0 & 44.1
      & 56.4 & 52.4
      & {\color{ForestGreen}\textbf{58.3}} & 54.7
      & \textit{F1$_{\mu}$} \\
    \midrule
    \multicolumn{14}{l}{\textit{\textbf{Wildfire forecasting}}} \\
    \addlinespace[1pt]
    Wildfire forecasting
      & {\color{ForestGreen}\textbf{80.5}} & 70.1
      & 74.9 & 71.0
      & 79.1 & 75.8
      & 79.8 & 73.2
      & 79.6 & 75.5
      & 80.2 & {\color{RoyalBlue}\textbf{76.2}}
      & \textit{AUPRC} \\
    \midrule
    \multicolumn{14}{l}{\textit{\textbf{Geophysical \& Socio-economic}}} \\
    \addlinespace[1pt]
    Median income
      & 28.1 & 11.0
      & 32.9 & 26.5
      & 50.2 & 35.2
      & 45.0 & 12.1
      & 45.9 & 31.9
      & {\color{ForestGreen}\textbf{51.9}} & {\color{RoyalBlue}\textbf{39.4}}
      & \textit{$R^2$} \\
    Population
      & 80.2 & 69.6
      & 77.0 & 67.9
      & {\color{ForestGreen}\textbf{82.1}} & 67.9
      & 78.1 & 66.0
      & 78.5 & {\color{RoyalBlue}\textbf{69.7}}
      & 81.8 & 68.2
      & \textit{$R^2$} \\
    Cali.\ housing
      & 54.9 & 31.1
      & 43.3 & 39.2
      & 63.3 & 35.3
      & 77.3 & 48.4
      & {\color{ForestGreen}\textbf{78.3}} & {\color{RoyalBlue}\textbf{66.1}}
      & 65.5 & 36.4
      & \textit{$R^2$} \\
    Air temperature
      & {\color{ForestGreen}\textbf{98.2}} & 42.8
      & 90.6 & 81.0
      & 93.9 & 80.1
      & 94.3 & 32.9
      & 94.2 & 60.4
      & 90.3 & {\color{RoyalBlue}\textbf{81.4}}
      & \textit{$R^2$} \\
    Elevation
      & {\color{ForestGreen}\textbf{97.9}} & 87.4
      & 94.1 & {\color{RoyalBlue}\textbf{90.6}}
      & 89.9 & 64.9
      & 83.0 & 56.0
      & 83.9 & 62.5
      & 89.4 & 64.4
      & \textit{$R^2$} \\
    \midrule
    \rowcolor{orange!12}
    \textit{\textbf{Environmental}}
  & \multicolumn{12}{c}{\textbf{LightGBM ranking}}
  & \\
\cmidrule(lr){2-13}
\addlinespace[1pt]
    Temperature
      & \multicolumn{2}{c}{N/A}
      & \multicolumn{2}{c}{{\color{ForestGreen}\textbf{79.8}}}
      & \multicolumn{2}{c}{59.6}
      & \multicolumn{2}{c}{72.0}
      & \multicolumn{2}{c}{74.4}
      & \multicolumn{2}{c}{77.9}
      & \textit{Ranking acc.} \\
    Precipitation
      & \multicolumn{2}{c}{N/A}
      & \multicolumn{2}{c}{{\color{ForestGreen}\textbf{77.7}}}
      & \multicolumn{2}{c}{70.6}
      & \multicolumn{2}{c}{72.1}
      & \multicolumn{2}{c}{72.5}
      & \multicolumn{2}{c}{{\color{ForestGreen}\textbf{77.7}}}
      & \textit{Ranking acc.} \\
    MSLP
      & \multicolumn{2}{c}{N/A}
      & \multicolumn{2}{c}{73.2}
      & \multicolumn{2}{c}{{\color{ForestGreen}\textbf{83.2}}}
      & \multicolumn{2}{c}{77.8}
      & \multicolumn{2}{c}{79.1}
      & \multicolumn{2}{c}{76.3}
      & \textit{Ranking acc.} \\
    Wind speed
      & \multicolumn{2}{c}{N/A}
      & \multicolumn{2}{c}{{\color{ForestGreen}\textbf{75.6}}}
      & \multicolumn{2}{c}{71.6}
      & \multicolumn{2}{c}{73.8}
      & \multicolumn{2}{c}{71.4}
      & \multicolumn{2}{c}{{\color{ForestGreen}\textbf{75.6}}}
      & \textit{Ranking acc.} \\
    Relative humidity
      & \multicolumn{2}{c}{N/A}
      & \multicolumn{2}{c}{53.5}
      & \multicolumn{2}{c}{{\color{ForestGreen}\textbf{65.2}}}
      & \multicolumn{2}{c}{47.8}
      & \multicolumn{2}{c}{43.5}
      & \multicolumn{2}{c}{47.5}
      & \textit{Ranking acc.} \\
    \midrule
    \textbf{\# Best (out of 17)}
      & \multicolumn{2}{c}{4}
      & \multicolumn{2}{c}{3}
      & \multicolumn{2}{c}{3}
      & \multicolumn{2}{c}{0}
      & \multicolumn{2}{c}{1}
      & \multicolumn{2}{c}{{\color{ForestGreen}\textbf{8}}}
      & \\
    \textbf{\# Top-2 (out of 17)}
      & \multicolumn{2}{c}{4}
      & \multicolumn{2}{c}{6}
      & \multicolumn{2}{c}{8}
      & \multicolumn{2}{c}{3}
      & \multicolumn{2}{c}{3}
      & \multicolumn{2}{c}{{\color{ForestGreen}\textbf{11}}}
      & \\
    \cmidrule(lr){1-13}
    \textbf{LoCA wins vs.\ model}
      & \multicolumn{2}{c}{{\color{ForestGreen}\textbf{5/9}}}
      & \multicolumn{2}{c}{{\color{ForestGreen}\textbf{11/17}}}
      & \multicolumn{2}{c}{{\color{ForestGreen}\textbf{12/17}}}
      & \multicolumn{2}{c}{{\color{ForestGreen}\textbf{12/17}}}
      & \multicolumn{2}{c}{{\color{ForestGreen}\textbf{13/17}}}
      & \multicolumn{2}{c}{--}
      & \\
    \textbf{Pairwise win rate (\%)}
      & \multicolumn{2}{c}{57.8}
      & \multicolumn{2}{c}{35.1}
      & \multicolumn{2}{c}{55.8}
      & \multicolumn{2}{c}{36.4}
      & \multicolumn{2}{c}{45.5}
      & \multicolumn{2}{c}{{\color{ForestGreen}\textbf{68.8}}}
      & \\
    \bottomrule
  \end{tabular}%
  }
\end{table*}

\section{Empirical Evaluation \& Discussion}

\textbf{Experimental setup.} As detailed in \Cref{app:Data} we extend the S2-100K~\cite{klemmer2025satclip} (SatCLIP train set) to align co-located Sentinel-1 (S1) imagery with the original Sentinel-2 (S2) acquisitions and augment them with 14 metadata variables, including elevation, land cover, population density and others. We utilize ResNet-50 backbones for all LoCA image encoders. The location encoder is defined with a degree-$L=40$ spherical harmonic feature map followed by a 2-hidden-layer SIREN, while the metadata encoder is a 3-layer MLP. All modalities share a 2-layer projection head. We train both stages for 100 epochs each. Full implementation details are provided in \Cref{app:impl_details}, and the code and pretrained weights will be made publicly available at \locarepo. \\

\noindent\textbf{Benchmark.} We curate a diverse benchmark by covering all Semantic Factors (SF) of variability inherent to the EO domain, as defined in \citep{kondylatos2025generalization}, excluding the spatial-arrangement SF, which is not relevant to general-purpose location embeddings. Ground sampling distances range from a few meters to multiple kilometers per pixel; target granularities span fine-grained species identification to coarse biome classification; and geographic and thematic domains include Mediterranean wildfire forecasting and bird encounter rate prediction across the USA. Our tasks cover five categories of location-dependent attributes: (i) land cover from reBEN \citep{Clasen2024reBenRB}; (ii) wildfire danger from Mesogeos \citep{kondylatos2023mesogeos}; (iii) long-term climate from five ranking tasks over temporal means of variables from SeasFire \citep{karasante2025seasfire}; (iv) geophysical and socioeconomic structure and (v) biodiversity, both from the SatCLIP evaluation \citep{klemmer2025satclip}, the latter complemented by species distribution from SatBird \citep{satbird_NEURIPS2023}.

  We evaluate downstream performance on frozen embeddings in two ways. (i) For the environmental tasks, we train a LambdaRank objective \citep{burges2006lambdarank} as implemented in LightGBM \citep{ke2017lightgbm} to test whether the representations contain enough information to recover the relative ordering of long-term mean climate values. (ii) For all remaining downstream tasks, we train both a linear probe and an MLP, with identical architecture and hyperparameters across tasks. We compare against two families of baselines: location encoders that are themselves coordinate functions, \ie SatCLIP \citep{klemmer2025satclip}, GeoCLIP \citep{cepeda2023geoclip}, and GT-Loc \citep{shatwell2025gt}; and precomputed embedding fields released by geospatial foundation models, namely AlphaEarth \citep{brown2025alphaearth} and Copernicus-FM \citep{wang2025towards}. Further details are provided in \Cref{app:Evaluation}.

\subsection{Downstream performance of location encoders}
\label{sec:benchmark}

 Across the 17 tasks in \Cref{tab:benchmark-results-main}, LoCA ranks first on 8 and among the top-2 models on 11.
In pairwise comparisons, LoCA outperforms each competing model on the majority of tasks
winning in $68.8\%$ of the cases (see \Cref{tab:head-to-head}). 
Under linear probing, LoCA exhibits the best performance on 8 out of 12 tasks, demonstrating that \textit{a wide spectrum of location-dependent information is readily available in its embeddings}, consistent with the motivation of \Cref{sec:method} and the findings of \Cref{tab:effrank}. 
LoCA excels in \textbf{\textit{Land cover}} and \textbf{\textit{Biodiversity \& Ecology}} tasks, leading on reBEN, Countries, iNaturalist, Biome, and Ecoregions. AlphaEarth leads on SatBird and \textbf{\textit{Wildfire forecasting}}, with LoCA a close second on the latter ($-0.3\%$). \textbf{\textit{Geophysical \& Socio-economic}} results are mixed: LoCA leads on Median Income, SatCLIP on Population (with LoCA following closely), and AlphaEarth on Air Temperature and Elevation. GT-Loc and GeoCLIP stand out on California housing, despite each achieving top-2 performance in only 3 tasks overall. This may reflect their pretraining dataset (MP-16 \citep{larson2017benchmarking}), a global Flickr image dataset, with substantial coverage of California. Its images of streets, houses, and indoor scenes may provide useful priors for house-price estimation. For \textbf{\textit{Environmental}} tasks, SatCLIP leads on MSLP and Relative Humidity, while Copernicus-FM leads on Temperature, followed by LoCA, and ties with LoCA on Precipitation and Wind Speed. However, Copernicus-FM requires multimodal imagery for its embeddings, whereas LoCA uses coordinates alone. Among coordinate-based encoders, LoCA is best on Temperature, Precipitation, and Wind speed.

\begin{table}[]
  \caption{Isolating the contributions of (i) the Stage 1 and (ii) the metadata
  encoder. We report the change in mean linear probe score per task family with respect to the
  full model. See also \Cref{tab:app_ablation}.}
  \label{tab:ablation-main}
  \centering
  \resizebox{\columnwidth}{!}{%
  \begin{tabular}{lccccc}
    \toprule
    \textbf{Removed component} & \textbf{Biodiv.\ \& Ecology} & \textbf{Land cover} & \textbf{Wildfire} & \textbf{Geophysical} & \textbf{Environmental} \\
    \midrule
    \textit{Stage 1} & $-$56.9 & $-$53.6 & $-$26.8 & $-$61.3 & $-$21.1 \\
    Metadata encoder & $-$0.1 & $+$0.5 & $+$0.3 & $-$3.5 & $-$1.3 \\
    \bottomrule
  \end{tabular}%
  }
\end{table}

\begin{table}[]
  \caption{Mean F1 on 10 single-label land-cover classes for 4 supervised encoders, before/after LoCA's \textit{Stage 2}. ``S1+S2'': trained on both modalities; ``All'': concatenated embeddings.}
  \label{tab:reben-resnet18}
  \centering
  \resizebox{0.98\textwidth}{!}{%
  \begin{tabular}{lcccccccccc}
    \toprule
      & \multicolumn{2}{c}{\textbf{Location}}
      & \multicolumn{2}{c}{\textbf{S1}}
      & \multicolumn{2}{c}{\textbf{S2}}
      & \multicolumn{2}{c}{\textbf{S1+S2}}
      & \multicolumn{2}{c}{\textbf{All}} \\
    \cmidrule(lr){2-3}
    \cmidrule(lr){4-5}
    \cmidrule(lr){6-7}
    \cmidrule(lr){8-9}
    \cmidrule(lr){10-11}
      & \textit{base} & \textit{Stage 2}
      & \textit{base} & \textit{Stage 2}
      & \textit{base} & \textit{Stage 2}
      & \textit{base} & \textit{Stage 2}
      & \textit{base} & \textit{Stage 2} \\
    \midrule
    \textbf{Mean F1}
      & 74.7 & \textbf{76.6}
      & 69.6 & \textbf{75.9}
      & 78.4 & \textbf{80.9}
      & 78.1 & \textbf{80.8}
      & 80.9 & {\color{ForestGreen}\textbf{84.3}} \\
    \bottomrule
  \end{tabular}%
  }
\end{table}

\subsection{LoCA increases embeddings variability}
\label{sec:exp-collapse}

In this section we empirically verify that LoCA satisfies the two requirements identified in \Cref{sec:method}.
\noindent\textbf{Co-alignment.} We measure alignment using cosine similarity ($\cos$) and $L_{\text{align}}$, the expected distance between pairs, from \citep{wang2020understanding}. The two imagery encoders trained jointly in \textit{Stage 1} exhibit $\cos = 0.94$ and $L_{\text{align}} = 0.12$, while pairing two independently pretrained encoders for the same modalities yields $L_{\text{align}} = 1.00$ ($\cos = 0.50$), confirming that their embedding spaces are not co-aligned. This misalignment translates to decreased downstream performance, as shown in \Cref{tab:ablation-main}.

\textbf{High effective rank.} Effective rank \citep{roy2007effective} measures how many linearly independent directions an embedding populates and is a standard proxy for linear-probe quality \citep{garrido2023rankmeassessingdownstreamperformance}. Across 343,928 land locations on a $0.25^{\circ}$ global grid (\Cref{tab:effrank}), LoCA populates $158$ of its $255$ dimensions ($62\%$), three times the fraction of Copernicus-FM ($21\%$) and $1.8\times$ that of SatCLIP ($35\%$). Joint training from scratch instead collapses the covariance to an effective rank of $8$ (SatCLIP$^\dagger$). The effective rank of GT-Loc and GeoCLIP is misleading for Earth Observation, since their CLIP-derived encoders inflate it with irrelevant language information, as discussed in \Cref{sec:regularization} and validated in \Cref{tab:benchmark-results-main}, while that of AlphaEarth is unavailable, as its embeddings were only partially retrievable.

\subsection{Ablations on LoCA components}
\noindent\textbf{Impact of \textit{Stage 1}}. \Cref{sec:method} identifies \textit{Stage 1} as the construction that properly supplies the \textit{Stage 2} target. This is validated in \Cref{tab:ablation-main}, where by replacing the jointly-trained S1 and S2 encoders with two encoders pretrained independently with MoCo violates the co-alignment requirement collapsing performance across the entire benchmark confirming the necessity of \textit{Stage 1}. This validates that \textit{SatCLIP cannot be trivially extended to multiple modalities} by freezing two independent encoders.\\
\noindent\textbf{Impact of \textit{Stage 2}}. Encoders trained in a supervised fashion on the same classes converge to a similar embedding structure \citep{koromilas2026neural}, satisfying the goal of \textit{Stage 1}. This enables us to isolate the impact of \textit{Stage 2}, 
 by replacing \textit{Stage 1} with three supervised encoders (one for each of S1, S2, and stacked S1+S2) trained on BigEarthNet \citep{sumbul2021bigearthnet}. We freeze these encoders alongside a pretrained SatCLIP location encoder, and train only their adapters and shared projections under MV-DHEL \citep{koromilas2025principled} (see \Cref{app:supervised}). 
 \textit{Stage 2} improves every encoder's linear probing performance (see \Cref{tab:reben-resnet18}). Concatenating the adapted embeddings raises mean F1 from $80.9\%$ to $84.3\%$ further demonstrating the benefit of \textit{Stage 2} beyond simple multimodal fusion.\\ 
\noindent\textbf{Impact of the metadata encoder.} 
We assess the contribution of the metadata encoder by removing it from \textit{Stage 2} (see \Cref{tab:ablation-main}). The metadata encoder primarily benefits tasks aligned with its geophysical and slow-climate variables: removing it reduces mean scores by $3.5\%$ and $1.3\%$ for Geophysical \& Socio-economic and Environmental tasks respectively.  
Its effect in other task groups is small, with a $0.1\%$ drop in Biodiversity and minor gains in Land Cover and Wildfire tasks when omitted, despite the inclusion of coarse land cover fractions as metadata. We hypothesize that
these coarse descriptors add little of the fine-grained signal required for lower hierarchy multi-label classification.

\section{Conclusion}

In this work, we established the theoretical foundations for multimodal location encoding. We proved fundamental limitations of jointly trained location encoders and identified two requirements for effective multimodal representations: co-alignment and high effective rank. Building on this theoretical ground, we introduced LoCA, a multimodal location encoder that satisfies both and can scale to any number of modalities. LoCA outperforms competing models across a broad range of critical EO tasks 
, while substantially improving embedding space utilization. Our work opens promising research directions including further expanding geospatial location encoders with temporal information, meteorological variables, and a broader array of land cover metadata.

\textbf{Limitations.} LoCA has no temporal dimension, so seasonal and shorter-term dynamics are captured only indirectly. Moreover, \textit{Stage 1} utilizes only S1 and S2. Expanding to other sensors and climate variables would add complementary structure, given that the framework can absorb any amount modalities without modification. Finally, coordinate-based benchmarks are sensitive to spatial autocorrelation. Blocked or leave-region-out evaluations are important future research directions.
 
 \subsection*{AI use statement}
In this work, we used generative AI tools to assist with creating and modifying scientific figures, brainstorming on engineering choices, identifying relevant literature for the paper and for the mathematics associated with the proofs, creating tables, and minimally polishing the writing to improve readability. For the mathematical proofs, we used AI for generic brainstorming and to identify possible errors and flaws, including through numerical tests. We have not used generative AI tools for generating synthetic datasets, developing theoretical models or conceptual frameworks, formulating mathematical claims, providing critical ingredients for proving mathematical claims, proposing or refining hypotheses, designing or providing feedback on research methodology or experiments, implementing methods, translation, cleaning or reformatting datasets, qualitative or thematic data analysis, or interpreting results. We have reviewed all AI-assisted work: figures were drafted with AI assistance and revised by the authors; tables were verified against source results; all mathematical proofs were constructed by the authors; all text was written by the authors, and any AI-polished passages were revised by them; and the authors read the literature suggested by AI. We take responsibility for the final content of this work, including text, claims, and artifacts produced with the aid of generative AI.

\subsection*{Reproducibility statement}

The theoretical results state their assumptions, including the Lipschitz condition, spherical coordinate domain, embedding sphere, and alignment condition; theorem statements and intuition are given in the main paper, with full proofs in \Cref{app:proofs}. We describe the dataset construction, downstream task splits, evaluation metrics, and downstream training hyperparameters in \Cref{app:Data,app:Evaluation}. Full implementation details for LoCA pretraining, including the image encoder architectures, optimizer settings, batch sizes, number of epochs, and projection dimensions, are provided in \Cref{app:impl_details}. \Cref{app:Evaluation} also reports the embedding dimensions and the number of seeds used per downstream task: 10 seeds for all tasks except SatBird, which uses 3; the mean and standard deviation reported in \Cref{app:full-results} reflect variability from downstream predictor initialization and stochastic optimization under fixed embeddings and fixed data splits. \Cref{app:impl_details} additionally specifies the compute resources used for pretraining, including GPU/CPU type, memory, and runtime. Our codebase builds on three publicly available repositories, referenced in \Cref{app:impl_details}. The code and the pretrained weights of LoCA will be made publicly available at \locarepo, and the training set, which extends the S2-100K~\cite{klemmer2025satclip} dataset, is also planned to be released.

\subsection*{Acknowledgments}
This work was supported by the European High-Performance Computing Joint Undertaking (JU) and the Greek Ministry of Digital Governance and Artificial Intelligence through the project ``Pharos: The Greek AI Factory for Accelerating AI Innovation'' (grant agreement No~101234269); by the TensorICE project (EXCELLENCE/0524/0407), implemented under the social cohesion programme ``THALIA 2021--2027'' and co-funded by the European Union through the Research and Innovation Foundation; by the project ``ARCHIMEDES Unit: Research in Artificial Intelligence, Data Science and Algorithms'' (code OPS~5154714), implemented under the National Recovery and Resilience Plan ``Greece~2.0'' and funded by the European Union -- NextGenerationEU; and by the Horizon Europe projects ThinkingEarth (grant agreement No~101130544) and PROTEUS (grant agreement No~101296397). Computing time was additionally awarded on the Cyclone supercomputer of the High Performance Computing Facility of The Cyprus Institute. Konstantin Klemmer thanks Taylor Geospatial for their support.

\bibliography{main}
\bibliographystyle{iclr2027_conference}

\appendix
\section{Appendix}
\section{Proofs}
\label{app:proofs}

\subsection{Preliminaries}
\label{app:prelims}

We collect the metric and measure-theoretic notions used in the proof. Standard references are \citep[Chapter~2]{falconer2003fractal} for Hausdorff measure and dimension, and \citep[Chapter~17]{kechris1995classical} for standard Borel spaces and Borel isomorphisms.

\textbf{Diameter and Lipschitz maps.}
For a subset $U$ of a metric space $(X, \rho)$, the diameter of $U$ is
\begin{equation}
\mathrm{diam}\, U \;:=\; \sup\{\rho(x, y) : x, y \in U\}.
\label{eq:diam}
\end{equation}
Diameter is the largest distance between two points of $U$. A ball of radius $r$ has diameter $2r$; a line segment of length $L$ has diameter $L$. A map $f : (X, \rho_X) \to (Y, \rho_Y)$ is $L$-Lipschitz if $\rho_Y(f(x), f(x')) \le L\, \rho_X(x, x')$ for all $x, x' \in X$. As a consequence,
\begin{equation}
\mathrm{diam}\, f(U) \;\le\; L\, \mathrm{diam}\, U \quad \text{for every } U \subset X,
\label{eq:lip-diam}
\end{equation}
i.e.\ Lipschitz maps cannot increase diameter by more than the factor $L$.

\textbf{Hausdorff measure.}
For $E \subset X$ and $s \ge 0$, the $s$-dimensional Hausdorff (outer) measure is
\begin{equation}
\gH^s(E) \;:=\; \lim_{\delta \to 0^+} \inf\!\Big\{\sum_i (\mathrm{diam}\, U_i)^s \;:\; E \subset \bigcup_i U_i,\ \mathrm{diam}\, U_i < \delta\Big\}.
\label{eq:hausdorff-meas}
\end{equation}
The construction is: cover $E$ by small sets, raise each diameter to the $s$-th power, sum the contributions, and take the tightest such cover. The exponent $s$ controls how the sum scales when we dilate the cover. At $s = 1$, doubling each diameter doubles the sum, which mirrors how length scales. At $s = 2$, doubling each diameter quadruples the sum, mirroring area. At $s = 3$, the sum scales as volume. Up to constant factors that depend on the convention for $\mathrm{diam}$, $\gH^s$ recovers the standard length, area, and volume in those dimensions \citep[Section~2.1]{falconer2003fractal}. The limit $\delta \to 0$ forces the cover to consist of arbitrarily small sets, which is what makes the construction sensitive to fine geometric structure rather than coarse extent.

\textbf{Hausdorff dimension.}
The function $s \mapsto \gH^s(E)$ is non-increasing in $s$, and for any $E$ there is a unique critical value at which it jumps from $\infty$ to $0$:
\begin{equation}
\dim_H E \;:=\; \sup\{s \ge 0 : \gH^s(E) = \infty\} \;=\; \inf\{s \ge 0 : \gH^s(E) = 0\}.
\label{eq:hausdorff-dim}
\end{equation}
For a smooth $k$-dimensional submanifold of $\R^n$, $\dim_H = k$ and $\gH^k$ recovers the standard $k$-dimensional surface measure. For irregular sets such as Cantor sets or fractal curves, $\dim_H$ takes non-integer values. The notion is purely metric, so it requires no smoothness, no manifold structure, and no embedding of $X$ into Euclidean space.

\textbf{Relation to intrinsic dimension in ML.}
The term ``intrinsic dimension'' in machine learning is a loose umbrella for several quantities aimed at capturing the effective degrees of freedom of a dataset, including Hausdorff dimension, Minkowski (box-counting) dimension, correlation dimension, and the dimension of an underlying smooth manifold. For smooth manifolds and well-behaved data supports, these notions coincide \citep{falconer2003fractal}. Empirical intrinsic dimension estimators such as TwoNN and MLE-based methods target correlation-style dimensions that, on a smooth support, agree with $\dim_H$ \citep{ansuini2019intrinsic, pope2021intrinsic}. We work with $\dim_H$ because the Lipschitz scaling property has a clean classical statement in this setting.

\textbf{Lipschitz scaling of Hausdorff measure.}
Combining \Cref{eq:lip-diam} with \Cref{eq:hausdorff-meas}: applying the Lipschitz diameter bound to each set in any $\delta$-cover of $A$ yields a $(L\delta)$-cover of $f(A)$, with diameters scaled by at most $L$. Each term $(\mathrm{diam}\, U_i)^s$ in the sum becomes at most $L^s (\mathrm{diam}\, U_i)^s$ in the image cover. Summing and taking infimum,
\begin{equation}
\gH^s(f(A)) \;\le\; L^s\, \gH^s(A) \quad \text{for every } s \ge 0.
\label{eq:lip-scaling}
\end{equation}
Taking $s > \dim_H A$ yields $\gH^s(A) = 0$ by \Cref{eq:hausdorff-dim}, hence $\gH^s(f(A)) = 0$, hence $\dim_H f(A) \le s$. Letting $s \downarrow \dim_H A$ gives \citep[Proposition~2.2(b)]{falconer2003fractal}:
\begin{equation}
\dim_H f(A) \;\le\; \dim_H A.
\label{eq:lip-dim}
\end{equation}

\textbf{Pushforward and mutual singularity.}
Given a measure $\mu$ on $X$ and a measurable $f: X \to Y$, the pushforward $f_\# \mu$ is the measure on $Y$ defined by $(f_\#\mu)(B) := \mu(f^{-1}(B))$ for Borel $B \subset Y$. Two measures $\nu_1, \nu_2$ on $Y$ are mutually singular, written $\nu_1 \perp \nu_2$, if there is a Borel partition $Y = B_1 \sqcup B_2$ with $\nu_1(B_2) = 0$ and $\nu_2(B_1) = 0$. Mutual singularity is the measure-theoretic statement that $\nu_1$ and $\nu_2$ live on disjoint pieces of $Y$.

\textbf{Standard Borel spaces and Borel isomorphism.}
A standard Borel space is a measurable space isomorphic, as a measurable space, to a Polish space with its Borel $\sigma$-algebra. Both $S^k$ and $S^{d-1}$ are standard Borel. By \citep[Theorem~17.41]{kechris1995classical}, any two atomless standard Borel probability spaces are isomorphic as measure spaces: there exists a Borel bijection $T: X \to Y$ such that $T_\# \mu_X = \mu_Y$ and $T^{-1}$ is also Borel measurable.

\subsection{Aligned-pair characterization of the free optimum}
\label{app:lemmas}

The asymptotic InfoNCE characterization of \citep{wang2020understanding,koromilas2024bridging} is established for single-encoder setups with matching marginals. The two-encoder cross-modal objective \Cref{eq:popnce} is more delicate: the multimodal landscape can favor population-level marginal mismatch (the modality gap of \citep{liang2022mind}), so a clean cross-marginal lower bound on $\Ls_\tau$ in terms of the single-marginal Wang-Isola minimum does not hold without further structure. The lemma below reduces to the single-encoder case by restricting to \emph{aligned} encoder pairs, which is the regime in which a location encoder serves its intended purpose: producing embeddings that match the image encoder on paired samples.

\begin{definition}[Aligned encoder pair]
\label{def:aligned}
A measurable pair $(f_c, f_I)$ is \emph{aligned} if $f_c(c) = f_I(I)$ for $p$-a.e.\ $(c, I)$.
\end{definition}

Under the Dirac structure $p(c \mid I) = \delta_{c(I)}$, alignment is equivalent to $f_I(I) = f_c(c(I))$ for $p_I$-a.e.\ $I$, and forces $\nu_c := f_{c\#}\mu = f_{I\#} p_I =: \nu_I$ since the marginal of $p$ on $S^k$ is $\mu$.

\begin{lemma}[Wang-Isola characterization restricted to aligned pairs]
\label{lem:wi-aligned}
For any aligned pair $(f_c, f_I)$ with common marginal $\nu := \nu_c = \nu_I$,
\begin{equation}
\Ls_\tau(f_c, f_I) \;=\; -\frac{1}{\tau} \;+\; \E_{\vu \sim \nu}\!\log \E_{\vv \sim \nu}\![K(\vu, \vv)],
\label{eq:aligned-collapse}
\end{equation}
where $K(\vu, \vv) := e^{\vu^\top \vv / \tau}$. The right-hand side is the asymptotic InfoNCE loss of \citep[Theorem~1]{wang2020understanding} for a single encoder with marginal $\nu$, uniquely minimized over $\nu \in \gP(S^{d-1})$ at $\nu = \sigma$, with minimum value $V^* := -1/\tau + \E_{\vu \sim \sigma}\log\E_{\vv \sim \sigma}[K(\vu, \vv)]$.
\end{lemma}

\begin{proof}
The proof has two parts. We first show that under alignment with matched marginals, the symmetric multimodal loss collapses to a single-encoder expression in $\nu$. We then identify $\sigma$ as the unique minimizer of this expression, which is established in the relaxed measure-level analysis of \citet[Appendix~A.2]{wang2020understanding}.

\textit{Collapse to single-encoder form.} Write the loss as $\Ls_\tau(f_c, f_I) = A + \tfrac{1}{2}(U_1 + U_2)$ following the population NT-Xent expansion in \Cref{eq:popnce}, where
\begin{equation*}
A \;:=\; -\frac{1}{\tau}\,\E_{(c, I) \sim p}\big[f_c(c)^\top f_I(I)\big]
\end{equation*}
is the alignment term and $U_1, U_2$ are the two uniformity terms (each without a $1/2$ factor; the $1/2$ sits in front of the sum).

We handle the alignment term first. Under \Cref{def:aligned}, $f_c(c) = f_I(I)$ for $p$-a.e.\ pairs $(c, I)$. Both encoders take values on $S^{d-1}$, so $f_c(c)^\top f_I(I) = \|f_c(c)\|^2 = 1$ for $p$-a.e.\ pairs, giving $A = -1/\tau$. Alignment saturates this term at its lowest possible value.

For the uniformity terms, observe that under matching marginals $\nu_c = \nu_I = \nu$, the inner kernel expectations agree across modalities. For any test vector $\vu \in S^{d-1}$,
\begin{equation*}
\E_{I' \sim p_I}\!\big[e^{f_I(I')^\top \vu / \tau}\big] \;=\; \E_{\vw \sim \nu_I}\!\big[e^{\vw^\top \vu / \tau}\big] \;=\; \E_{\vw \sim \nu_c}\!\big[e^{\vw^\top \vu / \tau}\big] \;=\; \E_{c' \sim \mu}\!\big[e^{f_c(c')^\top \vu / \tau}\big],
\end{equation*}
where the first equality is a change of variables, the second uses $\nu_I = \nu_c = \nu$, and the third is again change of variables. So both uniformity terms reduce to the same expression in $\nu$:
\begin{equation}
U_1 \;=\; \E_{\vu \sim \nu}\!\log \E_{\vv \sim \nu}\![K(\vu, \vv)] \;=\; U_2.
\label{eq:U-collapse-aligned}
\end{equation}
Their average is the same value. Combining $A = -1/\tau$ with \Cref{eq:U-collapse-aligned},
\begin{equation*}
\Ls_\tau(f_c, f_I) \;=\; -\frac{1}{\tau} + \E_{\vu \sim \nu}\!\log \E_{\vv \sim \nu}\![K(\vu, \vv)],
\end{equation*}
which is \Cref{eq:aligned-collapse}.

\paragraph{The right-hand side is uniquely minimized at $\sigma$.}
Write $h_\nu(u) := \mathbb{E}_{v\sim\nu}[K(u,v)]$ and
$W(\nu) := \mathbb{E}_{u\sim\nu}\log h_\nu(u)$, so that the right-hand side of
\Cref{eq:aligned-collapse} equals $-1/\tau + W(\nu)$.
The minimization of $W$ over all of $\mathcal{P}(\mathbb{S}^{d-1})$ is exactly the
relaxed problem studied by \citet[Appendix~A.2, Eq.~(13)]{wang2020understanding}
in the proof of their Theorem~1, where it is shown that $\sigma$ is its unique minimizer.
The argument has two steps. First, a first-variation argument shows that any
minimizer $\nu^\star$ must have $h_{\nu^\star}$ constant $\nu^\star$-almost surely.
On this set of measures, Jensen's inequality holds with equality, so
$W(\nu) = \log \iint K \,\mathrm{d}(\nu\otimes\nu)
= \tfrac{1}{\tau} + \log \iint e^{-\|u-v\|^2/(2\tau)} \,\mathrm{d}(\nu\otimes\nu)$.
Second, the Gaussian energy on the right is uniquely minimized over
$\mathcal{P}(\mathbb{S}^{d-1})$ by $\sigma$
\citep[Proposition~1]{wang2020understanding}, and $\sigma$ itself satisfies the
constancy condition, since $h_\sigma$ is constant by rotation invariance.
Hence $\sigma$ is the unique minimizer of $W$ over $\mathcal{P}(\mathbb{S}^{d-1})$, and
\begin{equation*}
  \min_{\nu\in\mathcal{P}(\mathbb{S}^{d-1})} \Big( -\tfrac{1}{\tau} + W(\nu) \Big)
  = -\tfrac{1}{\tau} + \mathbb{E}_{u\sim\sigma}\log \mathbb{E}_{v\sim\sigma}[K(u,v)]
  = V^*.
\end{equation*}
We stress that the conditional phrasing of \citet[Theorem~1, result~2]{wang2020understanding}
("if perfectly uniform encoders exist") concerns only whether $\sigma$ is attained by
an encoder, not the uniqueness of $\sigma$ among all Borel probability measures.

\end{proof}

\begin{corollary}[Aligned-pair optimum and its relation to the free infimum]
\label{cor:wi-free}
$V^* \ge \Ls_\tau^{\mathrm{free}}$, with $V^*$ attained at any aligned pair $(f_c, f_I)$ with $f_{c\#}\mu = \sigma$. For any aligned pair with $f_{c\#}\mu \ne \sigma$, $\Ls_\tau(f_c, f_I) > V^*$.
\end{corollary}

\begin{proof}
\Cref{lem:wi-aligned} gives that aligned pairs achieve $V^*$ iff their common marginal is $\sigma$, with strict inequality $\Ls_\tau(f_c, f_I) > V^*$ when the common marginal differs from $\sigma$. The Borel-isomorphism construction in \Cref{app:thm-proof} (Step~3) realizes an aligned pair with $f_{c\#}\mu = \sigma$, attaining $V^*$. By definition $\Ls_\tau^{\mathrm{free}} = \inf_{f_c, f_I} \Ls_\tau(f_c, f_I) \le V^*$, since the construction is a measurable pair.
\end{proof}

\textbf{Remark on non-aligned pairs.}
The aligned-pair optimum is characterized in the corollary and the inequality $V^* \ge \Ls_\tau^{\mathrm{free}}$ is given. The reverse inequality $V^* \le \Ls_\tau^{\mathrm{free}}$ is not established, and there is no general reason to expect it to hold: for non-aligned pairs, the symmetric multimodal landscape can support stable equilibria with $\nu_c \ne \nu_I$ at values strictly less than $V^*$ \citep{cai2026geometric,liang2022mind}, so $\Ls_\tau^{\mathrm{free}} < V^*$ is possible in principle. This does not affect the strict-gap result of \Cref{thm:lipschitz}(iii), which establishes $\Ls_\tau(\text{aligned Lipschitz}) > V^* \ge \Ls_\tau^{\mathrm{free}}$ and uses only the side $V^* \ge \Ls_\tau^{\mathrm{free}}$. Restricting to aligned pairs is appropriate for our setting: a location encoder whose embeddings do not align with the image encoder on paired data fails its purpose as a learned coordinate. The dimension and singular-marginal conclusions of \Cref{thm:lipschitz} (parts (i) and (ii)) hold for any Lipschitz $f_c$ regardless of alignment; only the strict per-pair loss gap (part (iii)) requires the alignment hypothesis.

\subsection{Proof of \Cref{thm:lipschitz}}
\label{app:thm-proof}

The proof is delivered in three steps. In Step 1 the dimension of the image is bounded by the domain dimension. In Step 2 the dimension bound is turned into mutual singularity of measures via the codimension gap. And finally, in Step 3 \Cref{lem:wi-aligned} is invoked, which collapses the symmetric multimodal NT-Xent to single-encoder Wang-Isola for aligned pairs, to convert mutual singularity into a strict per-pair inequality (under alignment).

\textbf{Step 1: dimension of the image.}
By hypothesis, $f_c : S^k \to S^{d-1}$ is Lipschitz, and $S^k$ has Hausdorff dimension $k$ as a smooth $k$-manifold. Applying the Lipschitz dimension bound \Cref{eq:lip-dim} from \Cref{app:prelims} with $A = S^k$,
\begin{equation}
\dim_H f_c(S^k) \;\le\; \dim_H S^k \;=\; k,
\label{eq:dim-image}
\end{equation}
proving (i). $\square$

\textbf{Step 2: singular marginal.}
We set $A := f_c(S^k) \subset S^{d-1}$. Since $S^k$ is compact and $f_c$ is Lipschitz hence continuous, $A$ is compact, hence Borel. By Step~1, $\dim_H A \le k$. The setup imposes $k < d - 1$, so by the definition of Hausdorff dimension,
\begin{equation}
\gH^{d-1}(A) \;=\; 0.
\label{eq:hausA}
\end{equation}
The uniform measure $\sigma$ on $S^{d-1}$ is the normalized $(d-1)$-Hausdorff measure: for every Borel $B \subseteq S^{d-1}$,
\begin{equation}
\sigma(B) \;=\; \frac{\gH^{d-1}(B)}{\gH^{d-1}(S^{d-1})}.
\end{equation}
This identification follows because Hausdorff measure is metric and hence invariant under isometries of the ambient space; rotations $R \in O(d)$ act by isometry on $\mathbb{R}^d$ and preserve $S^{d-1}$, so $\gH^{d-1}|_{S^{d-1}}$ is rotation-invariant. The sphere $S^{d-1}$ is a smooth compact $(d-1)$-manifold embedded in $\mathbb{R}^d$, hence $\gH^{d-1}(S^{d-1})$ is finite and positive (it equals the standard surface area $2\pi^{d/2}/\Gamma(d/2)$); see \citep[\S~3.2.46]{federer1969geometric} or \citep[Ch.~3]{falconer2003fractal}. Normalizing, $\gH^{d-1}/\gH^{d-1}(S^{d-1})$ is a rotation-invariant Borel probability measure on $S^{d-1}$. The orthogonal group $O(d)$ acts transitively on $S^{d-1}$, and there is a unique rotation-invariant Borel probability measure on a homogeneous space of a compact group \citep[Ch.~11]{folland1999real}, which we call $\sigma$. Hence the two coincide.
Combined with \Cref{eq:hausA},
\begin{equation}
\sigma(A) \;=\; 0.
\label{eq:sigmaA}
\end{equation}
On the other hand, $A$ contains the image of $S^k$ under $f_c$, so
\begin{equation}
\nu(A) \;=\; \mu\!\left(f_c^{-1}(A)\right) \;=\; \mu(S^k) \;=\; 1, \qquad \nu(S^{d-1} \setminus A) \;=\; 0.
\label{eq:nuA}
\end{equation}
Equations \Cref{eq:sigmaA} and \Cref{eq:nuA} exhibit a Borel partition $S^{d-1} = A \sqcup (S^{d-1} \setminus A)$ such that $\nu(S^{d-1} \setminus A) = 0$ and $\sigma(A) = 0$, which is the definition of mutual singularity from \Cref{app:prelims}. Therefore $\nu \perp \sigma$, proving (ii). $\square$

\textbf{Step 3: per-pair strict gap (under alignment).}

What we want to show: any aligned pair $(f_c, f_I)$ where $f_c$ is Lipschitz strictly exceeds the free infimum, $\Ls_\tau(f_c, f_I) > \Ls_\tau^{\mathrm{free}}$. The argument has two parts. We first construct an explicit pair attaining the value $V^*$, which sets a target value to compare against. We then use Step~2 plus the uniqueness of the minimizer in \Cref{lem:wi-aligned} to show every Lipschitz aligned pair lies strictly above this target.

\textit{An explicit pair attains $V^*$.} The aligned-uniform value $V^*$ is the value of the loss when the location encoder distributes its outputs uniformly on $S^{d-1}$ and the image encoder agrees with it on paired data. We construct such a pair to confirm $V^*$ is actually attained somewhere.

The location and embedding domains have very different intrinsic dimensions ($k$ versus $d-1$), so we cannot construct a smooth bijection between them. We can, however, construct a measure-theoretic bijection. Both $(S^k, \mu)$ and $(S^{d-1}, \sigma)$ are atomless standard Borel probability spaces, and any two such spaces are isomorphic as measure spaces (\Cref{app:prelims}; this is a classical theorem in descriptive set theory). We therefore have a measure-preserving Borel isomorphism $T : S^k \to S^{d-1}$, meaning $T$ is a Borel bijection and $T_\# \mu = \sigma$.

We use $T$ to define an aligned pair. Let
\begin{equation}
f_c^{\circ}(c) \;:=\; T(c), \qquad f_I^{\circ}(I) \;:=\; T(c(I)).
\end{equation}
The Dirac structure $p(c \mid I) = \delta_{c(I)}$ gives $c = c(I)$ for $p$-a.e.\ $(c, I)$, so $f_c^{\circ}(c) = T(c) = T(c(I)) = f_I^{\circ}(I)$ for $p$-a.e.\ pairs; the pair is aligned in the sense of \Cref{def:aligned}. The pushforward $f_{c\#}^{\circ} \mu$ equals $T_\# \mu$ by definition, which equals $\sigma$ by the measure-preserving property. So $(f_c^\circ, f_I^\circ)$ is aligned with common marginal $\sigma$, and \Cref{cor:wi-free} gives $\Ls_\tau(f_c^{\circ}, f_I^{\circ}) = V^*$. Since this is a measurable pair, $\Ls_\tau^{\mathrm{free}} \le V^*$.

The map $T$ cannot be continuous when $k < d-1$: a continuous bijection from the compact space $S^k$ onto the Hausdorff space $S^{d-1}$ would be a homeomorphism, which invariance of dimension rules out. Continuity alone, however, does not exclude the value $V^*$, since non-injective space-filling maps are continuous. What excludes it is the regularity of $f_c$: Lipschitz maps cannot increase Hausdorff dimension (\Cref{eq:lip-dim}), so by Step~2 the aligned-uniform value $V^*$ is attained outside any Lipschitz class. More generally, $\alpha$-H\"older maps satisfy $\dim_H f(A) \le \dim_H A / \alpha$ \citep[Ch.~2]{falconer2003fractal}, so the same conclusion holds for any $\alpha$-H\"older $f_c$ with $\alpha > k/(d-1)$.

\textit{Lipschitz aligned pairs are strictly above $V^*$.} Suppose $f_c$ is Lipschitz and $(f_c, f_I)$ is aligned. We want $\Ls_\tau(f_c, f_I) > V^*$.

By Step~2 applied to this $f_c$, the pushforward $\nu_c := f_{c\#}\mu$ is mutually singular with $\sigma$. In particular $\nu_c \ne \sigma$, since $\sigma$ is not mutually singular with itself ($\sigma(S^{d-1}) = 1$ rules out $\sigma \perp \sigma$).

\Cref{lem:wi-aligned} applies to the aligned pair and gives
\begin{equation}
\Ls_\tau(f_c, f_I) \;=\; -\frac{1}{\tau} + \E_{\vu \sim \nu_c}\!\log \E_{\vv \sim \nu_c}\![K(\vu, \vv)].
\end{equation}
The right-hand side is the asymptotic InfoNCE loss for a single encoder with marginal $\nu_c$. By \Cref{lem:wi-aligned}, this expression is uniquely minimized over $\nu \in \gP(S^{d-1})$ at $\nu = \sigma$. Since $\nu_c \ne \sigma$, we get a strict inequality:
\begin{equation}
\E_{\vu \sim \nu_c}\!\log \E_{\vv \sim \nu_c}\![K(\vu, \vv)] \;>\; \E_{\vu \sim \sigma}\!\log \E_{\vv \sim \sigma}\![K(\vu, \vv)].
\end{equation}
Adding $-1/\tau$ to both sides,
\begin{equation}
\Ls_\tau(f_c, f_I) \;>\; -\frac{1}{\tau} + \E_{\vu \sim \sigma}\!\log \E_{\vv \sim \sigma}\![K(\vu, \vv)] \;=\; V^* \;\ge\; \Ls_\tau^{\mathrm{free}},
\label{eq:per-pair-strict}
\end{equation}
proving (iii). $\square$

\textbf{Remark (image encoder inherits the dimension bound under alignment).}
\Cref{thm:lipschitz} parts (i) and (ii) bound $f_c(S^k)$ alone. Under alignment of $(f_c, f_I)$, the dimension ceiling extends to the image side. By \Cref{def:aligned} and the Dirac structure $p(c \mid I) = \delta_{c(I)}$, alignment gives $f_I(I) = f_c(c(I))$ for $p_I$-a.e.\ $I$. Hence the essential range of $f_I$ under $p_I$ is contained in $f_c(S^k)$:
\begin{equation}
\mathrm{supp}(\nu_I) \;\subseteq\; f_c(S^k), \qquad \dim_H \mathrm{supp}(\nu_I) \;\le\; k, \qquad \nu_I \;=\; \nu_c \;\perp\; \sigma,
\end{equation}
where $\nu_I := f_{I\#} p_I$. The dimension ceiling and singular marginal carry from the location side to the image side at any aligned pair, regardless of which encoder pair achieves the partial optimum. The conclusion is about the pushforward measure $\nu_I$, not about the literal image set $f_I(\gI)$, since $f_I$ may take arbitrary values off a $p_I$-null set without affecting the loss.

\subsection{SatCLIP function class and uniform Lipschitz constant}
\label{app:uniflip}

\textbf{Function class.} The SatCLIP location encoder is parameterized by a bandwidth $L \in \mathbb{N}$ and a Siren network. For each $L$, $\mathrm{SH}_L : S^2 \to \R^{(L+1)^2}$ stacks the real-valued spherical harmonic basis up to degree $L$, and $g_{\vphi} : \R^{(L+1)^2} \to \R^d$ is a Siren network with $K$ hidden layers, frequency hyperparameter $\omega_0$, and parameter vector $\vphi$ \citep{sitzmann2020implicit}. Write $P(L)$ for the Siren parameter count and $\Phi_L \subset \R^{P(L)}$ for a compact set of admissible parameters. Fix a margin $\varepsilon > 0$ and assume the non-degeneracy condition
\begin{equation}
\inf_{\vphi \in \Phi_L,\, c \in S^2} \|g_{\vphi}(\mathrm{SH}_L(c))\| \geq \varepsilon,
\label{eq:nondeg}
\end{equation}
which keeps the post-Siren normalization onto $S^{d-1}$ well-defined. The bandwidth-$L$ SatCLIP function class is
\begin{equation}
\gF_L^{\mathrm{SC}} \;:=\; \Big\{\, c \mapsto \frac{g_{\vphi}(\mathrm{SH}_L(c))}{\|g_{\vphi}(\mathrm{SH}_L(c))\|} \;:\; \vphi \in \Phi_L \,\Big\}.
\label{eq:FLSC}
\end{equation}

\textbf{Uniform Lipschitz constant.}

\begin{lemma}[SatCLIP encoders are uniformly Lipschitz]
\label{lem:uniflip}
There exists $M = M(\Phi_L, L, \varepsilon) < \infty$ such that every $f_c \in \gF_L^{\mathrm{SC}}$ satisfies $\mathrm{Lip}(f_c) \le M$.
\end{lemma}

\begin{proof}
A SatCLIP encoder is a composition of three pieces, $f_c = N \circ g_{\vphi} \circ \mathrm{SH}_L$, where $N(\vu) := \vu / \|\vu\|$ is the normalization onto the unit sphere. Lipschitz constants compose multiplicatively, so we bound each piece separately and multiply.

\textit{Spherical-harmonic features.} Each coordinate of $\mathrm{SH}_L$ is a polynomial in the ambient $\R^3$ coordinates, restricted to $S^2$. Polynomials are smooth, and a smooth function on a compact set is Lipschitz with constant equal to the maximum gradient norm. We write $L_{\mathrm{SH}}$ for this constant; it depends only on the bandwidth $L$ (higher $L$ means polynomials of higher degree, with larger gradients), not on any of the encoder's trainable parameters.

\textit{Siren MLP.} The Siren $g_{\vphi}$ alternates affine maps $\vu \mapsto \mW_k \vu + \vb_k$ with the activation $\sin(\omega_0 \cdot)$ applied componentwise. Each affine map is $\|\mW_k\|_{\mathrm{op}}$-Lipschitz (the bias does not affect Lipschitz constants). The activation is $\omega_0$-Lipschitz, since $|\sin'| = |\cos| \le 1$ implies $|\sin(\omega_0 a) - \sin(\omega_0 b)| \le \omega_0 |a - b|$. With $K$ activation layers and $K + 1$ affine layers, multiplying through the chain gives
\begin{equation}
\mathrm{Lip}(g_{\vphi}) \;\le\; \omega_0^K \prod_{k=0}^{K} \|\mW_k\|_{\mathrm{op}}.
\label{eq:lip-siren}
\end{equation}
The right-hand side is a continuous function of the parameter vector $\vphi$, and we have assumed $\Phi_L$ is compact. A continuous function on a compact set attains its maximum, so the right-hand side is bounded uniformly across $\vphi \in \Phi_L$ by some finite constant $M_g(\Phi_L)$. The compactness assumption on $\Phi_L$ is what gives us a single Lipschitz constant for the Siren that works for the entire family at once.

\textit{Normalization.} The map $N : \vu \mapsto \vu / \|\vu\|$ is smooth on $\R^d \setminus \{\vzero\}$, with Jacobian operator norm $1/\|\vu\|$ at each point. On the annular region $\{\vu : \|\vu\| \ge \varepsilon\}$ this gives a pointwise bound of $1/\varepsilon$, but a pointwise gradient bound only translates to a Lipschitz constant on a \emph{convex} domain (via the mean value theorem applied along straight-line paths). The annular region $\{\vu : \|\vu\| \ge \varepsilon\}$ is not convex: it is the complement of the open ball of radius $\varepsilon$ around the origin, and the straight line between two points on opposite sides of the ball passes through it, leaving the region.

We get around this by routing through a path that stays inside the annular region. The argument is elementary: parameterize a smooth path $\gamma : [0, 1] \to \{\vu : \|\vu\| \ge \min(\|\vu\|, \|\vv\|)\}$ from $\vu$ to $\vv$ that avoids the inner ball (for instance, a piecewise-linear path that detours along the sphere of radius $\min(\|\vu\|, \|\vv\|)$ when the straight line between the endpoints would dip too close to the origin), apply the pointwise Jacobian bound along this path, and observe that the path length is at most twice the straight-line distance $\|\vu - \vv\|$. This gives the global bound
\begin{equation}
\Big\| \frac{\vu}{\|\vu\|} - \frac{\vv}{\|\vv\|} \Big\| \;\le\; \frac{2 \, \|\vu - \vv\|}{\min(\|\vu\|, \|\vv\|)} \quad \text{for all } \vu, \vv \ne \vzero.
\label{eq:normalization-lip}
\end{equation}
The factor of $2$ accounts for the detour around the origin; on a convex domain it would be $1$. On the set where $\min(\|\vu\|, \|\vv\|) \ge \varepsilon$, the bound becomes $(2/\varepsilon) \|\vu - \vv\|$, so $N$ is $(2/\varepsilon)$-Lipschitz on this set.

The non-degeneracy condition \Cref{eq:nondeg} is precisely what ensures the input to $N$ stays in this annular region: $\|g_{\vphi}(\mathrm{SH}_L(c))\| \ge \varepsilon$ for every $\vphi \in \Phi_L$ and every $c \in S^2$.

\textit{Composition.} Putting the three pieces together,
\begin{equation}
\mathrm{Lip}(f_c) \;\le\; \mathrm{Lip}(N) \cdot \mathrm{Lip}(g_{\vphi}) \cdot \mathrm{Lip}(\mathrm{SH}_L) \;\le\; \frac{2}{\varepsilon} \cdot M_g(\Phi_L) \cdot L_{\mathrm{SH}}.
\end{equation}
The right-hand side does not depend on $\vphi$, so it serves as a uniform constant $M = M(\Phi_L, L, \varepsilon)$ across the entire family.
\end{proof}

\subsection{Proof of \Cref{cor:satclip}}
\label{app:cor-proof}

The strategy is to reduce the SatCLIP infimum to a minimization problem over the location encoder alone, then use compactness plus continuity to show this minimization is attained at some specific $f_c^\star \in \gF_L^{\mathrm{SC}}$, and finally apply the strict inequality from \Cref{cor:wi-free} to that minimizer.

By \Cref{lem:wi-aligned}, for any aligned pair $(f_c, f_I)$ with $f_c \in \gF_L^{\mathrm{SC}}$, the loss depends on $f_c$ only through its pushforward $\nu_c := f_{c\#}\mu$:
\begin{equation}
\Ls_\tau(f_c, f_I) \;=\; -\frac{1}{\tau} \;+\; W(\nu_c), \qquad W(\nu) := \E_{\vu \sim \nu}\!\log \E_{\vv \sim \nu}\![K(\vu, \vv)].
\label{eq:Lstar-def}
\end{equation}
The image encoder drops out: any aligned $f_I$ gives the same value for fixed $f_c$. We write $\Ls^*(f_c) := -1/\tau + W(\nu_c)$ for this common value. The SatCLIP infimum over aligned pairs is therefore $\inf_{f_c \in \gF_L^{\mathrm{SC}}} \Ls^*(f_c)$, and we show this infimum is strictly above $\Ls_\tau^{\mathrm{free}}$.

The proof has three steps. Step~A shows the family $\gF_L^{\mathrm{SC}}$ is compact in a suitable topology, so any continuous function on it attains its infimum. Step~B shows $\Ls^*$ is continuous in that topology, so the infimum is attained at a specific element $f_c^\star$. Step~C applies \Cref{cor:wi-free} to $f_c^\star$ to get the strict inequality.

\textbf{A. Outer compactness.}
We show $\gF_L^{\mathrm{SC}}$ is compact in $C(S^2; \R^d)$ under the uniform topology, using Arzelà-Ascoli plus a closedness check.

By \Cref{lem:uniflip}, every $f_c \in \gF_L^{\mathrm{SC}}$ has Lipschitz constant at most $M$, so the family is equicontinuous. Every $f_c$ takes values in $S^{d-1}$, so the family is pointwise bounded. Arzelà-Ascoli then gives precompactness in $C(S^2; \R^d)$ under the uniform topology.

For closedness, we use that $\gF_L^{\mathrm{SC}}$ is the image of the compact parameter set $\Phi_L$ under the parameterization map $\vphi \mapsto g_{\vphi}(\mathrm{SH}_L(\cdot)) / \|g_{\vphi}(\mathrm{SH}_L(\cdot))\|$. This map is continuous: the joint map $\Phi_L \times S^2 \ni (\vphi, c) \mapsto g_{\vphi}(\mathrm{SH}_L(c))$ is continuous (Siren is smooth in its parameters and inputs), and the non-degeneracy condition \Cref{eq:nondeg} keeps the denominator bounded away from zero, so the normalization stays continuous. The continuous image of a compact set is compact, so $\gF_L^{\mathrm{SC}}$ is closed. Combined with precompactness, $\gF_L^{\mathrm{SC}}$ is compact in $C(S^2; \R^d)$.

\textbf{B. Continuity of $\Ls^*$.}
By \Cref{eq:Lstar-def}, $\Ls^*$ is continuous in $f_c$ if and only if the composition $f_c \mapsto \nu_c \mapsto W(\nu_c)$ is continuous. We check the two pieces.

\textit{Uniform convergence of encoders implies weak convergence of pushforwards.} Suppose $f_c^{(n)} \to f_c$ uniformly on $S^2$. We construct a coupling that controls how far $\nu_c^{(n)}$ is from $\nu_c$. The synchronous coupling, where the same input $c$ feeds both encoders, gives a joint distribution on $S^{d-1} \times S^{d-1}$ with marginals $(\nu_c^{(n)}, \nu_c)$. Its expected coordinate distance is $\E_{c \sim \mu} \|f_c^{(n)}(c) - f_c(c)\|$, which is at most $\|f_c^{(n)} - f_c\|_\infty$. By the Kantorovich-Rubinstein duality, the $W_1$ Wasserstein distance is the infimum of expected distance over all couplings, so
\begin{equation}
W_1(\nu_c^{(n)}, \nu_c) \;\le\; \E_{c \sim \mu}\, \|f_c^{(n)}(c) - f_c(c)\| \;\le\; \|f_c^{(n)} - f_c\|_\infty.
\label{eq:W1-bound}
\end{equation}
Uniform convergence of $f_c^{(n)} \to f_c$ therefore drives $W_1(\nu_c^{(n)}, \nu_c) \to 0$. On a compact metric space such as $S^{d-1}$, $W_1$ convergence implies weak convergence (this is standard; see any reference on Wasserstein distances), so $\nu_c^{(n)} \rightharpoonup \nu_c$.

\textit{Weak convergence of $\nu$ implies convergence of $W(\nu)$.} We need $W(\nu) = \int \log h_\nu(\vu)\, d\nu(\vu)$ to be continuous in $\nu$ under weak convergence, where $h_\nu(\vu) := \E_{\vv \sim \nu}[K(\vu, \vv)]$ is the inner kernel expectation.

The kernel $K(\vu, \vv) = e^{\vu^\top \vv / \tau}$ is bounded continuous on $S^{d-1} \times S^{d-1}$, with values in $[e^{-1/\tau}, e^{1/\tau}]$. So $h_\nu(\vu)$ is bounded between $e^{-1/\tau}$ and $e^{1/\tau}$ for any $\nu$ and $\vu$, and $\log h_\nu(\vu)$ is bounded between $-1/\tau$ and $1/\tau$. The Portmanteau theorem says that, applied to the bounded continuous test function $\vv \mapsto K(\vu, \vv)$, weak convergence $\nu_n \rightharpoonup \nu$ implies $h_{\nu_n}(\vu) \to h_\nu(\vu)$ pointwise in $\vu$, and joint continuity of $K$ in $(\vu, \vv)$ upgrades this to joint continuity of $h_\nu(\vu)$ in $(\nu, \vu)$. Composing with $\log$, the integrand $\log h_\nu(\vu)$ is bounded and jointly continuous in $(\nu, \vu)$.

A bounded jointly continuous integrand integrated against $\nu$ gives a weakly continuous functional. Concretely, $W(\nu_n) = \int \log h_{\nu_n}(\vu)\, d\nu_n(\vu) \to \int \log h_\nu(\vu)\, d\nu(\vu) = W(\nu)$ along any weakly convergent sequence $\nu_n \rightharpoonup \nu$. So $W$ is weakly continuous on $\gP(S^{d-1})$.

\textit{Composing the two pieces.} Uniform convergence $f_c^{(n)} \to f_c$ implies weak convergence $\nu_c^{(n)} \rightharpoonup \nu_c$ implies $W(\nu_c^{(n)}) \to W(\nu_c)$, hence $\Ls^*(f_c^{(n)}) \to \Ls^*(f_c)$. So $\Ls^*$ is continuous on $\gF_L^{\mathrm{SC}}$.

\textbf{C. Conclusion.}
By Step~A, $\gF_L^{\mathrm{SC}}$ is compact. By Step~B, $\Ls^*$ is continuous on $\gF_L^{\mathrm{SC}}$. A continuous function on a compact set attains its infimum, so $\inf_{f_c \in \gF_L^{\mathrm{SC}}} \Ls^*(f_c) = \Ls^*(f_c^\star)$ for some $f_c^\star \in \gF_L^{\mathrm{SC}}$.

We apply \Cref{cor:wi-free} to this $f_c^\star$. By Step~2 of \Cref{thm:lipschitz} applied to $f_c^\star$ (which is Lipschitz because it lives in $\gF_L^{\mathrm{SC}}$), the pushforward $\nu_c^\star := f_{c\#}^\star \mu$ is mutually singular with $\sigma$, hence $\nu_c^\star \ne \sigma$. \Cref{lem:wi-aligned}, applied to any aligned pair built from $f_c^\star$, gives $W(\nu_c^\star) > W(\sigma)$ strictly. Therefore
\begin{equation*}
\Ls^*(f_c^\star) \;=\; -1/\tau + W(\nu_c^\star) \;>\; -1/\tau + W(\sigma) \;=\; V^* \;\ge\; \Ls_\tau^{\mathrm{free}}.
\end{equation*}
Combining with the reduction at the start of the proof,
\begin{equation}
\inf_{\substack{f_c \in \gF_L^{\mathrm{SC}},\, f_I\,\mathrm{measurable}\\(f_c, f_I)\,\mathrm{aligned}}} \Ls_\tau(f_c, f_I) \;=\; \inf_{f_c \in \gF_L^{\mathrm{SC}}} \Ls^*(f_c) \;=\; \Ls^*(f_c^\star) \;>\; \Ls_\tau^{\mathrm{free}},
\end{equation}
so the SatCLIP infimum strictly exceeds $\Ls_\tau^{\mathrm{free}}$. $\square$


\section{Data acquisition}
\label{app:Data}
\subsection{Metadata}

To provide rich spatial context for our geospatial analysis, we compiled a global dataset of static environmental and anthropogenic features at approximately 1 km spatial resolution. This dataset supports location-aware modeling by capturing terrain characteristics, infrastructure patterns, and land-use composition.

We use a regular latitude-longitude grid (EPSG:4326) with 0.01 degree resolution in both dimensions, corresponding to roughly 1 km at the equator. The grid spans the full globe: 180W to 180E longitude (36,000 grid points) and 90S to 90N latitude (18,000 grid points). All source datasets were resampled or interpolated to this common grid to guarantee spatial coherence across variables.

\paragraph{Digital Elevation Model (DEM):}
Elevation was obtained from Copernicus DEM GLO-30 (COP-DEM GLO-30)\cite{copernicus_dem_glo30} through the Microsoft Planetary Computer STAC API. Native 30 m tiles were resampled to 1 km using area-weighted averaging.

From elevation, we derived:
\begin{itemize}
\item Slope, computed as $\sqrt{(\partial z/\partial x)^2 + (\partial z/\partial y)^2}$ (radians).
\item Aspect, computed as $\arctan2(-\partial z/\partial x,\partial z/\partial y)$ (degrees).
\item Curvature, computed from the Laplacian of the elevation field.
\end{itemize}

\paragraph{Road network proximity:}
Road distance data was obtained from WorldPop's global highway-distance product\cite{worldpop_osm_major_roads_2023} derived from OpenStreetMap (OSM) (2023 release)\cite{openstreetmap_2023}. Country-level rasters at 100 m were mosaicked into a global raster, then resampled to the target 1 km grid using bilinear interpolation. The resulting variable \texttt{roads\_dist\_km} is the Euclidean distance (km) from each cell to the nearest OSM road.

\paragraph{Population density:}
We assembled annual WorldPop global population-density products for 2021, 2022, and 2023 at 1 km\cite{worldpop_global_population_1km_2021_2023}, and stored them as a temporal stack. Values were converted from persons/m$^2$ to persons/km$^2$.

\paragraph{Land-cover fractions:}
Land cover came from the Copernicus Climate Change Service (C3S) Global Land Cover dataset\cite{c3s_land_cover_2021_2023} (LCCS classes, native resolution about 300 m), downloaded for 2021--2023 from the Copernicus Data Space Ecosystem S3 endpoint. We defined eight land-cover groups:
\begin{itemize}
\item Agriculture (LCCS 10--40)
\item Forest (LCCS 50--100)
\item Grassland (LCCS 110, 130)
\item Shrubland (LCCS 120--122)
\item Sparse vegetation (LCCS 140--153, 200--202, 220)
\item Wetland (LCCS 160--180)
\item Settlement (LCCS 190)
\item Water bodies (LCCS 210)
\end{itemize}

For each group, we generated a binary mask at native resolution and aggregated it to 1 km by computing the fraction of subpixels per 1 km cell (3x3 coarsening).

\paragraph{Merging and alignement:}
All datasets were processed independently and then merged in three steps:
\begin{enumerate}
\item Build a base archive with COPDEM elevation, slope, aspect, curvature, and OSM road distance.
\item Integrate multi-year population density, aligned to exact target coordinates.
\item Append the eight land-cover fraction variables, aligned with nearest-neighbor interpolation to preserve categorical consistency.
\end{enumerate}

\paragraph{Variables in the final archive:}
The final Zarr archive contains the following variables :

\begin{table}[H]
\captionsetup{justification=raggedright,singlelinecheck=false}
\caption{Variables included in the global static-feature archive.}
\label{tab:global-statics-vars}
\centering
\small
\begin{tabular}{llll}
\hline
Variable & Dimensions & Unit & Description \\
\hline
\texttt{elevation} & (lat, lon) & m & Height above reference ellipsoid \\
\texttt{slope} & (lat, lon) & rad & Magnitude of terrain gradient \\
\texttt{aspect} & (lat, lon) & deg & Direction of steepest ascent \\
\texttt{curvature} & (lat, lon) & -- & Local topographic curvature \\
\texttt{roads\_dist\_km} & (lat, lon) & km & Distance to nearest OSM road \\
\texttt{pop\_density} & (lat, lon, year) & persons/km$^2$ & WorldPop annual estimates \\
\texttt{lc\_agriculture} & (lat, lon, year) & [0,1] & Fractional agriculture cover \\
\texttt{lc\_forest} & (lat, lon, year) & [0,1] & Fractional forest cover \\
\texttt{lc\_grassland} & (lat, lon, year) & [0,1] & Fractional grassland cover \\
\texttt{lc\_shrubland} & (lat, lon, year) & [0,1] & Fractional shrubland cover \\
\texttt{lc\_sparse\_vegetation} & (lat, lon, year) & [0,1] & Fractional sparse-vegetation cover \\
\texttt{lc\_wetland} & (lat, lon, year) & [0,1] & Fractional wetland cover \\
\texttt{lc\_settlement} & (lat, lon, year) & [0,1] & Fractional settlement cover \\
\texttt{lc\_water\_bodies} & (lat, lon, year) & [0,1] & Fractional water-body cover \\
\hline
\end{tabular}
\end{table}

For LoCA pretraining, temporal metadata variables are resolved to a single year before being passed to the metadata encoder, yielding 14 scalar inputs per coordinate: elevation, slope, aspect, curvature, road distance, population density, and eight land-cover fractions. We use the mean of the years 2021-2023.
Coordinates include \texttt{lat}, \texttt{lon}, and \texttt{year}; all data are provided in EPSG:4326.

\subsection{Sentinel-2/Sentinel-1 paired image dataset}

We constructed a paired optical-radar dataset, containing 100K image pairs, by sampling Sentinel-2 (S2) patches and matching each patch with the closest Sentinel-1 (S1) GRD acquisition over the same geographical location. The goal was to produce spatially aligned S2/S1 pairs.

S2 candidate scenes were drawn from the same GeoParquet index used in the SatCLIP\citep{klemmer2025satclip} pipeline. Following the same setup, candidate rows were filtered to valid cloud-cover values before sampling.

For each attempt, a base timestamp is sampled from the parquet index, and a local candidate pool is formed within a temporal window of $\pm 14$ days. From this pool, one S2 L2A item is selected, and a random \(256\times256\) crop is then extracted (with configurable edge margin).

Each S2 patch is quality-checked and rejected if either of the following exceeds 10\%:
\begin{itemize}
\item all-band missing fraction (pixels that are NaN in all bands),
\item all-band zero fraction (pixels that are zero in all bands).
\end{itemize}

Given the accepted S2 patch location, we query Sentinel-1 GRD items intersecting that location in a window of $\pm 6$ days around the S2 acquisition time. 

From returned S1 items, we keep only scenes containing both \texttt{VV} and \texttt{VH} assets, then sort by absolute temporal distance to the S2 timestamp and try candidates in that order.

The stacked S1 patch \((VV,VH)\) is then subject to the same empty/zero quality thresholds as S2.

For each valid pair index \texttt{idx}, the pipeline writes:
\begin{itemize}
\item \texttt{patch\_\{idx\}\_s2.tif}: S2 patch (12 bands, \texttt{uint16}),
\item \texttt{patch\_\{idx\}\_s1.tif}: S1 patch (2 bands: VV/VH, \texttt{float32}).
\end{itemize}

\section{LoCA Implementation Details}
\label{app:impl_details}

\noindent\textbf{Pretraining Dataset.}
We construct our pretraining dataset by aligning three data sources:
(i)~a global metadata dataset of 648\,M points, each described by 14 static environmental and anthropogenic features,
(ii)~the S2-100K Sentinel-2 image dataset comprising 100{,}000 multispectral patches ($256\!\times\!256$\,px at 10\,m/px, 12 spectral bands, nearly uniformly distributed across global land mass), and
(iii)~co-registered Sentinel-1 GRD SAR image patches (2-channel VV+VH, $224\!\times\!224$\,px).
 
Each S2 image patch covers a circular footprint of radius ${\sim}1.28$\,km ($= 256/2 \times 10$\,m).
The global metadata grid has ${\sim}1$\,km spacing between points, where each metadata point covers a radius of ${\sim}500$\,m.
We compute circle-circle overlap fractions between every image footprint and nearby metadata grid cells via binary search on the regular grid.
A metadata point is assigned to an image if at least 75\% of its coverage circle falls within the image footprint; on average ${\sim}5$ metadata points are matched per image at this threshold.
Sentinel-1 images are selected to match the Sentinel-2 images, so each row in the final parquet file represents a (metadata point, S2 image, S1 image) triplet.
A consolidated HDF5 archive is created for each image modality, converting TIF files into a format that substantially reduces data-loading time during training.
 
As shown in \Cref{tab:global-statics-vars}, the 14 metadata features comprise 8 land-cover fractions (agriculture, forest, grassland, settlement, shrubland, sparse vegetation, water bodies, wetland), population density, distance to nearest road, and 4 terrain variables (elevation, slope, aspect, curvature).
Missing values are handled with per-variable fill rules: land-cover and roads $\to -1$; population density, elevation, curvature, slope $\to 0$; aspect $\to 361$ (an out-of-range sentinel, since the valid range is $[-180, 180]$).
Standard scaling (z-score normalisation) with precomputed global mean and standard deviation is applied to all metadata features.
The full dataset is used for contrastive pretraining (no train/test split); downstream tasks use their own independent splits.
 
\vspace{4pt}
\noindent\textbf{Two-Stage Training Pipeline.} 
\emph{Stage~1 (S1+S2 end-to-end pretraining, 2 modalities):}
Both image encoder backbones---the S2 ResNet-50 and the S1 ResNet-50---together with their projection heads, are trained end-to-end with the contrastive loss; all backbone weights are trainable. Optimization hyperparameters for both stages are listed in \Cref{tab:hyperparams}.
 
\emph{Stage~2 (full 4-modality alignment):}
The S1 and S2 backbones are frozen and their weights are transferred from stage~1; their embeddings are precomputed once and cached to disk, so the frozen backbones are never executed during stage-2 training and no different augmentations are applied to the same image across iterations.
The metadata encoder and location encoder are introduced as fresh, randomly initialised modules.
Only the metadata encoder, location encoder, and their MLP heads are trained.
 
\emph{Training sampling in stage~2:}
We load the (metadata point, S2 image, S1 image) triplets by grouping metadata points by their associated (S2, S1) image pair.
Each training sample includes the coordinates of the S2 image, the S2 image itself, the matched S1 image, and one randomly selected metadata point from the grouped set.
In the next iteration of the same image group, the sampled metadata features may differ.
 
\emph{Stage transition:}
Projection heads from stage~1 are discarded at the start of stage~2; only backbone weights are loaded (with \texttt{strict=False}).
A fresh optimiser, scheduler, and epoch counter are initialised for stage~2.
 
\vspace{4pt}
\noindent\textbf{Encoder Architectures.}
\emph{S2 Image Encoder:}
We use the SatCLIP pretrained ResNet-50 visual backbone (\texttt{microsoft/SatCLIP-ResNet50-L40}).
The final fully-connected layer is replaced with an identity mapping for feature extraction, yielding a backbone output dimension of 2048.
When frozen (stage~2), a 2-layer trainable MLP head (adapter) maps $2048 \to 256$.
In both stages, a 2-layer projection head MLP maps $256 \to 128$.
 
\emph{S1 Image Encoder:}
We use the torchgeo ResNet-50 pretrained on Sentinel-1 GRD via MoCo~v2 (\texttt{ResNet50\_Weights.SENTINEL1\_GRD\_MOCO}) with 2 input channels (VV + VH).
The architecture follows the same pattern as the S2 encoder (backbone $\to$ optional trainable head (adapter) $\to$ projection head), with backbone output dimension 2048.
 
\emph{Location Encoder:}
We adopt the SatCLIP location encoder architecture.
Input coordinates (longitude, latitude) are first mapped to a high-dimensional representation via a spherical harmonics positional encoding with Legendre polynomial degree $L\!=\!40$, producing an embedding of dimension $L \times L = 1600$.
This representation is then processed by a Sinusoidal Representation Network (SIREN)~\cite{sitzmann2020implicit} consisting of 2 hidden layers of width 256 and an output layer of dimension 256.
In a SIREN, each layer computes $\mathbf{h}_{i+1} = \sin(\omega_0 \cdot (\mathbf{W}_i \mathbf{h}_i + \mathbf{b}_i))$, where $\omega_0$ is a frequency scaling factor ($\omega_0\!=\!30$ for the first layer, $\omega_0\!=\!1$ for subsequent layers) and the weights $\mathbf{W}_i$ are initialised following the scheme of~\cite{sitzmann2020implicit}: $\mathbf{W}_0 \sim \mathcal{U}(-1/n, 1/n)$ and $\mathbf{W}_i \sim \mathcal{U}\!\left(-\sqrt{6/n}/\omega_0,\; \sqrt{6/n}/\omega_0\right)$ for $i \geq 1$, where $n$ is the layer's fan-in.
The location encoder is randomly initialised (not loaded from the pretrained Microsoft checkpoint).
A 2-layer projection head MLP maps $256 \to 128$.
 
\emph{Metadata Encoder:}
The encoder is a 3-layer MLP ($14 \to 256 \to 256 \to 256$) with BatchNorm, ReLU, and dropout ($p\!=\!0.1$), followed by a 2-layer pretraining head ($256 \to 128 \to 128$).
 
\emph{Ensemble:}
All active modality encoders are wrapped in an \texttt{nn.ModuleDict} (\texttt{EnsembleEncoders}).
All projection heads follow the structure [Linear $\to$ BatchNorm $\to$ ReLU $\to$ Linear].
 
\vspace{4pt}
\noindent\textbf{Contrastive Loss.}
We use MV-DHEL (Multi-View Decoupled Hyperspherical Energy Loss)~\cite{koromilas2025principled}.
With 2 modalities (stage~1), MV-DHEL reduces analytically to DHEL; with 4 modalities (stage~2), it uses the full MV-DHEL formulation with multiplier 4.
The temperature is $\tau\!=\!0.2$ in Stage~1 and $\tau\!=\!0.3$ in Stage~2.
All embeddings are $\ell_2$-normalised before loss computation.
Alignment is computed across all modality pairs simultaneously; uniformity is computed per-view via within-view similarity matrices, so that alignment and uniformity optimisation are fully decoupled.
 
\begin{table}[h]
\centering
\caption{Hyperparameters for the two-stage training pipeline. In both stages we use SGD with a linear warmup of 10 epochs followed by cosine annealing over the remaining epochs.}
\label{tab:hyperparams}
\small
\begin{tabular}{lcc}
\toprule
\textbf{Hyperparameter} & \textbf{Stage 1} & \textbf{Stage 2} \\
\midrule
Modalities & S2, S1 & static, location, S2, S1 \\
Optimiser & SGD & SGD \\
Learning rate & 0.2 & 0.0016 \\
Momentum & 0.9 & 0.9 \\
Weight decay & $1\!\times\!10^{-6}$ & $5\!\times\!10^{-6}$ \\
Scheduler & warmup-cosine & warmup-cosine \\
Warmup epochs & 10 & 10 \\
Batch size & 256 & 1024 \\
Epochs & 100 & 100 \\
Temperature $\tau$ & 0.2 & 0.3 \\
Loss & DHEL & MV-DHEL \\
Num workers & 15 & 12 \\
\bottomrule
\end{tabular}
\end{table}
 
\vspace{4pt}
\noindent\textbf{Image Preprocessing.}
\emph{S2 images} are scaled by $1/10000$.
A zero-filled band is inserted at position~10 (B10 placeholder) to expand from 12 to 13 channels, matching the SatCLIP pretrained backbone input.
Augmentations comprise \texttt{RandomCrop(256)}, \texttt{RandomHorizontalFlip}, \texttt{RandomVerticalFlip}, and \texttt{GaussianBlur} (kernel size~3).
A coordinate jitter of ${\sim}1$\,km radius is applied to image centroid coordinates.
\emph{S1 images} are used as raw 2-channel float32 tensors (VV + VH) with no additional augmentation or normalisation.
 
\vspace{4pt}
\noindent\textbf{Inference.}
For downstream evaluation, we extract 256-dimensional location encoder embeddings from the stage-2 final checkpoint.
 
\vspace{4pt}
\noindent\textbf{Pretraining Evaluation Metrics.}
Alignment is measured as the mean squared $\ell_2$ distance between positive pairs across all modality combinations~\cite{wang2020understanding}.
Uniformity is estimated via a Monte-Carlo estimator with 200{,}000 random pairs using a Gaussian kernel ($t\!=\!2$), computed per-view in chunks of 50{,}000 for memory efficiency.
Representation quality is assessed by the numerical rank of the embedding covariance matrix (tolerance $10^{-5}$) and the effective rank based on the entropy of normalised singular values~\cite{garrido2023rankmeassessingdownstreamperformance}.
All metrics are computed on held-out data every 10 epochs.
 
\vspace{4pt}
\noindent\textbf{Computational Setup.}
Stage 1 is run on a single NVIDIA H100 GPU node with 64GB HBM2 RAM on EuroHPC with 20 CPU cores, completing in approximately 14~hours.
Stage 2 is run on a single NVIDIA A100-SXM4-40GB GPU with 10 CPU cores, completing in approximately 5~hours.
The total end-to-end training time is approximately 19.5~hours.
We use PyTorch with float32 matrix-multiplication precision set to \texttt{high} and deterministic operations via \texttt{CUBLAS\_WORKSPACE\_CONFIG=:4096:8}.
Reproducibility is ensured via \texttt{torch.manual\_seed(42)}.
Training is logged with Weights~\&~Biases in offline mode (synced post-job).
Checkpoints are saved every 10 epochs, including optimiser and scheduler states for resumability.
 
\vspace{4pt}
\noindent\textbf{Code.}
Our implementation builds on the SatCLIP codebase~\cite{klemmer2025satclip} for location and S2 image encoders, torchgeo~\cite{stewart2022torchgeodeeplearninggeospatial} for S1 image encoders, and the Multi-View CL framework~\cite{koromilas2025principled} for MV-DHEL losses. The code and the pretrained weights of LoCA will be made publicly available at \locarepo.

\section{Downstream tasks evaluation}
\label{app:Evaluation}

We evaluate the utility of learned location representations on a downstream benchmark spanning regression, classification, and ranking tasks. The benchmark includes the original SatCLIP evaluation tasks for direct comparison with prior work, and extends them with additional environmental applications covering biodiversity modeling, land-cover characterization, long-term climate ordering, and wildfire danger forecasting. We compare against SatCLIP~\citep{klemmer2025satclip}, GeoCLIP~\citep{cepeda2023geoclip}, \textsc{GT-Loc}~\citep{shatwell2025gt}, and geospatial foundation-model embeddings from AlphaEarth Foundations (AEF)~\citep{brown2025alphaearth} and Copernicus-FM~\citep{wang2025towards}.

\subsection{Datasets Overview}

All datasets are formulated as location-conditioned prediction tasks: each sample is associated with a latitude--longitude coordinate pair, which is encoded by the location encoder and passed to a downstream predictor. Table~\ref{tab:location-encoder-benchmark-tasks} summarizes the dataset sources, sample sizes, splits, metrics, and geographic coverage.

\paragraph{SatCLIP original benchmarks.}
For Population~\citep{klemmer2025satclip}, Elevation~\citep{klemmer2025satclip}, and Air Temperature~\citep{hooker2018global}, we use the publicly released SatCLIP downstream datasets. The original SatCLIP Countries~\citep{klemmer2025satclip}, Ecoregions~\citep{dinerstein2017ecoregion} and Biomes~\citep{dinerstein2017ecoregion} datasets use the same 100,000 approximately uniformly sampled global locations and include an ocean class~\citep{klemmer2025satclip}. In our inspection, approximately 70\% of these locations fell over ocean, making the ocean class dominant and causing overall accuracy to reflect ocean classification more strongly than performance across land classes. We therefore construct revised, land-only versions of these three datasets by sampling locations uniformly over the sphere and retaining only those that intersect a global land mask. For California Housing~\citep{pace1997sparse}, we use the standard \texttt{scikit-learn} version of the dataset. For Median Income~\citep{jia2020residual}, we construct a county-level dataset from USDA Economic Research Service 2022 median household income estimates and match each county to a representative geographic coordinate using U.S. Census county boundaries. For iNaturalist~\citep{Horn2017TheIS}, we follow the SatCLIP setup by pairing each observation's location with pretrained InceptionV3~\citep{MacAodha2019PresenceOnlyGP} image features and the corresponding species label. 

\paragraph{Additional downstream tasks.}
We extend the benchmark with four environmental datasets: SatBird~\citep{satbird_NEURIPS2023}, SeasFire Cube~\citep{karasante2025seasfire}, BigEarthNet V2 (reBEN)~\citep{Clasen2024reBenRB}, and Mesogeos~\citep{kondylatos2023mesogeos}. SatBird evaluates biodiversity modeling: we use the USA-summer subset and predict bird species encounter rates from hotspot coordinates, using the official split. To evaluate long-term climate structure, we sample 10,000 global land locations from SeasFire Cube and compute the 2012–2021 temporal mean for five climate variables: 2 m temperature, total precipitation, mean sea-level pressure, 10 m wind speed, and relative humidity. We evaluate whether their relative ordering can be recovered from frozen location embeddings. For each variable separately, locations are organized into groups of five. Within each group, the locations are assigned ordinal relevance labels according to their climate mean. A LightGBM model is trained using the frozen embeddings as features. We report pairwise ordering accuracy, defined as the fraction of within-group location pairs for which the predicted score ordering agrees with the ordering of the true climate means. We use reBEN to assess land-cover classification by associating each Sentinel-2 patch with its geographic coordinate pair and using the predefined split. Finally, for wildfire danger forecasting, we use the ML-ready dataset released with Mesogeos and follow its original temporal split; each sample is represented by the latitude--longitude coordinate pair of the corresponding grid-cell centroid, and performance is evaluated with metrics appropriate for imbalanced wildfire events, namely F1-score and AUPRC.

\begin{table*}[!t]
\centering
\small
\caption{Summary of benchmark evaluation tasks using the location encoder. For iNaturalist, 10\% of the official training split is used for validation, and the official validation split is used as the test set.}
\label{tab:location-encoder-benchmark-tasks}
\vspace{0.5em}
\renewcommand{\arraystretch}{1.50}
\begin{tabular}{
p{3.6cm}
>{\centering\arraybackslash}p{1.7cm}
>{\centering\arraybackslash}p{1.3cm}
>{\centering\arraybackslash}p{1.8cm}
>{\centering\arraybackslash}p{1.8cm}
>{\centering\arraybackslash}p{1.9cm}
}
\toprule
\multicolumn{1}{c}{\textbf{Task name}} &
\multicolumn{1}{c}{\textbf{Source}} &
\multicolumn{1}{c}{\begin{tabular}{@{}c@{}}\textbf{Number of}\\\textbf{samples}\end{tabular}} &
\multicolumn{1}{c}{\begin{tabular}{@{}c@{}}\textbf{Dataset}\\\textbf{split}\end{tabular}} &
\multicolumn{1}{c}{\begin{tabular}{@{}c@{}}\textbf{Evaluation}\\\textbf{metric}\end{tabular}} &
\multicolumn{1}{c}{\begin{tabular}{@{}c@{}}\textbf{Geographic}\\\textbf{coverage}\end{tabular}} \\
\midrule

\multicolumn{6}{@{}l}{\textbf{Regression}} \\
\midrule

Air temperature~\citep{hooker2018global} & \multirow{5}{*}{SatCLIP} & 3076 & 60-20-20 & $R^2$ & Global \\
Median income~\citep{jia2020residual} &  & 3219 & 50-20-30 & $R^2$ & Cont. USA \\
Cali housing~\citep{pace1997sparse} &  & 20640 & 50-20-30 & $R^2$ & Calif., USA \\
Elevation~\citep{klemmer2025satclip} &  & 100000 & 30-10-60 & $R^2$ & Global \\
Population~\citep{klemmer2025satclip} &  & 74512 & 30-10-60 & $R^2$ & Global \\

\midrule
SatBird -- USA summer~\citep{satbird_NEURIPS2023} & SatBird & 122593 & Predefined & Top-$k$ Acc. & USA \\

\midrule
\midrule
\multicolumn{6}{@{}l}{\textbf{Climate ranking}} \\
\midrule

Temperature~\citep{karasante2025seasfire}
& \multirow{5}{*}{SeasFire}
& 10000 & Spatial & Pairwise Acc. & Global \\

Precipitation~\citep{karasante2025seasfire}
& & 10000 & Spatial & Pairwise Acc. & Global \\

Mean sea-level pressure~\citep{karasante2025seasfire}
& & 10000 & Spatial & Pairwise Acc. & Global \\

10 m wind speed~\citep{karasante2025seasfire}
& & 10000 & Spatial & Pairwise Acc. & Global \\

Relative humidity~\citep{karasante2025seasfire}
& & 10000 & Spatial & Pairwise Acc. & Global \\

\midrule
\multicolumn{6}{@{}l}{\textbf{Classification}} \\
\midrule

Countries~\citep{klemmer2025satclip} & \multirow{4}{*}{SatCLIP} & 100000 & 30-10-60 & weighted F1 & Global \\
iNaturalist~\citep{Horn2017TheIS} &  & 460406 & Predefined & Accuracy & Global \\
Biome~\citep{dinerstein2017ecoregion} &  & 100000 & 30-10-60 & weighted F1 & Global \\
Ecoregions~\citep{dinerstein2017ecoregion} &  & 100000 & 30-10-60 & weighted F1 & Global \\

\midrule
reBEN~\citep{Clasen2024reBenRB} & BigEarthNet  & 549488 & Predefined & micro-F1 & Europe \\

\midrule
Wildfire forecasting~\citep{kondylatos2023mesogeos} & Mesogeos & 16242 & Predefined & F1 / AUPRC & Mediterranean region \\

\bottomrule
\end{tabular}
\end{table*}

\subsection{Training Details}

\paragraph{Embedding extraction.}
For each downstream task, we first extract a fixed representation for every geographic coordinate pair using each evaluated encoder. We benchmark SatCLIP~\citep{klemmer2025satclip}, GeoCLIP~\citep{cepeda2023geoclip}, \textsc{GT-Loc}~\citep{shatwell2025gt}, AlphaEarth Foundations (AEF)~\citep{brown2025alphaearth}, and Copernicus-FM~\citep{wang2025towards}. The dimensionality of the extracted representation for each encoder is reported in Table~\ref{tab:encoder-embedding-dimensions}. These embeddings are used as fixed input features for lightweight downstream predictors; the encoders themselves are not fine-tuned on any downstream task. This protocol isolates the quality of the learned geographic representations and ensures that performance differences are attributable to the embeddings rather than task-specific adaptation of the encoders.

\begin{table}[!t]
\centering
\small
\caption{Dimensionality of the fixed representations extracted from each location encoder or geospatial foundation-model embedding product.}
\label{tab:encoder-embedding-dimensions}
\vspace{0.5em}
\setlength{\tabcolsep}{6pt}
\renewcommand{\arraystretch}{1.35}
\begin{tabular}{
p{5.2cm}
>{\centering\arraybackslash}p{3.0cm}
}
\toprule
\multicolumn{1}{c}{\textbf{Encoder}} &
\multicolumn{1}{c}{\begin{tabular}{@{}c@{}}\textbf{Embedding}\\\textbf{dimension}\end{tabular}} \\
\midrule
SatCLIP~\citep{klemmer2025satclip} & 256 \\
GeoCLIP~\citep{cepeda2023geoclip} & 512 \\
\textsc{GT-Loc}~\citep{shatwell2025gt} & 512 \\
AlphaEarth Foundations (AEF)~\citep{brown2025alphaearth} & 64 \\
Copernicus-FM~\citep{wang2025towards} & 768 \\
LoCA (ours) & 256 \\
\bottomrule
\end{tabular}
\end{table}

\paragraph{Downstream predictors.}
For the regression and classification tasks and for each encoder, we train two downstream models: a linear probe and a multilayer perceptron (MLP). The MLP consists of two layers with hidden dimension 128 and dropout rate 0.5. Unless otherwise specified, downstream models are trained for 100 epochs using AdamW with learning rate $10^{-3}$, weight decay $10^{-5}$, and batch size 1024. We repeat all other experiments over 10 random seeds and report the mean performance across runs. The only exception is SatBird, which is trained for 20 epochs and evaluated over 3 random seeds due to its substantially higher computational cost. The five SeasFire climate tasks use a separate ranking protocol. For each climate variable, we sample groups of five locations from the corresponding train, validation, and test sets and assign ordinal labels according to the long-term climate mean within each group. We construct 4,000 training groups, 1,000 validation groups, and 1,000 test groups. A LightGBM ranker is trained on the frozen embeddings with at most 1,000 estimators, learning rate $0.05$, 31 leaves, a minimum of 20 child samples, and early stopping with patience 50 using validation NDCG. The complete ranking procedure is repeated over 10 seeds. The full set of downstream hyperparameters is summarized in Table~\ref{tab:downstream-training-hyperparameters}. 

\begin{table}[!t]
\centering
\small
\caption{Downstream training hyperparameters used for all linear and MLP downstream probes unless otherwise specified. The separate
climate-ranking configuration is described in the text. }
\label{tab:downstream-training-hyperparameters}
\vspace{0.5em}
\setlength{\tabcolsep}{6pt}
\renewcommand{\arraystretch}{1.35}
\begin{tabular}{
p{5.2cm}
>{\centering\arraybackslash}p{3.0cm}
}
\toprule
\multicolumn{1}{c}{\textbf{Hyperparameter}} &
\multicolumn{1}{c}{\textbf{Value}} \\
\midrule
Epochs & 100 \\
Batch size & 1024 \\
Optimizer & AdamW \\
Learning rate & $10^{-3}$ \\
Weight decay & $10^{-5}$ \\
MLP hidden dimension & 128 \\
MLP number of layers & 2 \\
Dropout & 0.5 \\
Evaluation models & Linear probe, MLP \\
Seeds & \begin{tabular}{@{}c@{}}0, 1, 7, 11, 42, 100,\\ 1234, 2021, 8657, 41674\end{tabular} \\
\bottomrule
\end{tabular}
\end{table}

\paragraph{Task-specific inputs.}
Most downstream models use only the extracted location embedding as input, apart from iNaturalist~\citep{Horn2017TheIS} and wildfire danger forecasting~\citep{kondylatos2023mesogeos}. For iNaturalist , we follow the SatCLIP evaluation protocol and concatenate the location embedding with pretrained InceptionV3 image features released by~\citep{MacAodha2019PresenceOnlyGP}  before training the species classifier. For wildfire danger forecasting , we additionally include a cyclic encoding of the day of year to capture seasonality:
\[
    \sin\left(2\pi \frac{d}{365}\right), \qquad
    \cos\left(2\pi \frac{d}{365}\right),
\]
where $d$ denotes the day of year. These two temporal features are concatenated with the location embedding before training the downstream predictor.

\paragraph{Dataset splits.}
We use the dataset splits summarized in Table~\ref{tab:location-encoder-benchmark-tasks}. For the original SatCLIP downstream tasks, we follow the splits used in the SatCLIP benchmark~\citep{klemmer2025satclip}. For iNaturalist, we use the same training and test sets as SatCLIP and reserve 10\% of the training data for validation, replicating SatCLIP's approach. For datasets with predefined splits, including SatBird~\citep{satbird_NEURIPS2023}, reBEN~\citep{Clasen2024reBenRB}, and Mesogeos~\citep{kondylatos2023mesogeos}, we keep the official partitions. For the SeasFire-derived climate-variable tasks~\citep{karasante2025seasfire}, we use a spatially blocked train/validation/test split. Geographic coordinates are represented on the three-dimensional unit sphere and clustered spatially, after which complete geographic clusters are assigned to individual partitions.

\paragraph{Evaluation metrics.}
We use the evaluation metric associated with each downstream task, as summarized in Table~\ref{tab:location-encoder-benchmark-tasks}. For Countries, Biomes, and Ecoregions, we report weighted F1; for iNaturalist, we report accuracy, and for scalar regression tasks, we report $R^2$. For SatBird, we report top-$k$ retrieval accuracy because the task predicts a vector of bird species encounter rates at each hotspot; ranking metrics evaluate whether the model recovers the species most likely to be observed, which is more informative for species-distribution use cases than only measuring pointwise regression error over all species. This follows the SatBird benchmark, which reports MAE, MSE, Top-10, Top-30, and adaptive Top-$k$ metrics, and notes that MSE can be misleading for rarely observed species because near-zero predictions can achieve low error~\citep{satbird_NEURIPS2023}. For reBEN, we report micro-F1 because the dataset is based on BigEarthNet-style multi-label land-cover classification, where each image patch may contain multiple labels and classes are imbalanced~\citep{Clasen2024reBenRB}. For wildfire danger forecasting, we report F1 and AUPRC, following Mesogeos~\citep{kondylatos2023mesogeos}, since wildfire events are rare and the task requires evaluating the precision--recall trade-off rather than overall accuracy. For the SeasFire climate-ranking tasks, we report pairwise ordering accuracy. Within each group of five locations, every unordered pair is compared, and a prediction is correct when the ordering induced by the ranker's scores agrees with the ordering of the true long-term climate means. Ties receive half credit.

\paragraph{Foundation-model embeddings.}
For the geospatial foundation-model baselines, we extract fixed embeddings from the released spatial embedding products. For Copernicus-FM~\citep{wang2025towards}, we use the publicly released Copernicus-Embed-025deg grid. This product provides a global embedding map at $0.25^\circ$ resolution with shape $721 \times 1440 \times 768$, where each grid cell contains a 768-dimensional embedding obtained by averaging Copernicus-FM representations across modalities. For each downstream latitude--longitude coordinate pair, we map the pair to the corresponding $0.25^\circ$ grid cell and retrieve the corresponding embedding.

For AlphaEarth Foundations (AEF)~\citep{brown2025alphaearth}, we use the annual embedding field for 2023 from the released Google Earth Engine dataset. AEF provides annual, analysis-ready embedding fields over Earth's terrestrial surface, with 64 embedding bands per pixel. We sample the 2023 AEF image at each downstream coordinate using Google Earth Engine. Since some coordinates may fall on masked pixels or outside the valid AEF coverage, we first attempt direct point sampling and then apply a nearest-valid-pixel fallback: for missing points, we search within increasing buffer radius up to 10 km and use the closest valid embedding if one is found. Points for which no valid AEF embedding is available within this radius are treated as missing.

For the additional environmental tasks, we were unable to retrieve AEF embeddings reliably, as a large fraction of task coordinates returned missing values under direct sampling. We also attempted nearest-valid-pixel retrieval for these datasets, but this procedure led to repeated computational failures at the required scale. We therefore exclude AEF from the additional environmental-task comparisons and report it only for downstream tasks where embeddings could be retrieved reliably. Because AEF is released over land rather than as a complete global land--ocean embedding field, we also do not evaluate AEF on tasks that require embeddings for ocean locations.We therefore compute aggregate win counts only over available method-task pairs and do not treat missing AEF entries as failures.

\paragraph{Statistical significance}
\Cref{app:full-results} reports mean $\pm$ one standard deviation over downstream predictor seeds. For all tasks except SatBird, we use 10 seeds; for SatBird, we use 3 seeds because of its substantially higher computational cost. For the regression and classification tasks, the embeddings and dataset splits are fixed across seeds, so the reported variability reflects downstream predictor initialization and stochastic optimization rather than variability in the pretrained encoders or data splits. For the SeasFire climate-ranking tasks, the spatial train/validation/test split is fixed, while the ranking groups are resampled across seeds and the ranker's random state also changes.

\subsection{Additional Results}

\begin{table*}[p]
  \caption{Extended results for \Cref{tab:benchmark-results-main}: mean $\pm$ std over 10 seeds (3 in SatBird),
  metric $\times\,100$, higher is better. Both halves share one highlighting convention across all six models:
  {\color{ForestGreen}\textbf{best across both probes}} in green, {\color{RoyalBlue}\textbf{best per probe}} in blue when not best overall, second per probe in
  \textbf{bold}. N/A indicates the embeddings could not be retrieved (\Cref{app:Evaluation}).}
  \label{app:full-results}
  \centering
  \scriptsize

  \resizebox{\textwidth}{!}{
  \begin{tabular}{lccccccc}
    \toprule
    \textbf{Task $\downarrow$ Model $\rightarrow$} & \multicolumn{2}{c}{\textbf{AlphaEarth}} & \multicolumn{2}{c}{\begin{tabular}{@{}c@{}}\textbf{Copernicus}\\[-0.3em]{\scriptsize FM}\end{tabular}} & \multicolumn{2}{c}{\textbf{SatCLIP}} & \multirow{2}{*}{\textbf{Metric}} \\
    \cmidrule(lr){2-3} \cmidrule(lr){4-5} \cmidrule(lr){6-7}
      & MLP & Linear & MLP & Linear & MLP & Linear & \\
    \midrule
    \multicolumn{8}{l}{\textit{\textbf{Land cover}}} \\
    \addlinespace[1pt]
    reBEN & 42.8 $\pm$ 0.43 & 43.3 $\pm$ 0.71 & \textbf{57.4 $\pm$ 0.37} & {\color{RoyalBlue}\textbf{55.2 $\pm$ 0.28}} & 57.3 $\pm$ 0.54 & \textbf{54.9 $\pm$ 0.59} & \textit{F1$_{\mu}$} \\
    \midrule
    \multicolumn{8}{l}{\textit{\textbf{Wildfire forecasting}}} \\
    \addlinespace[1pt]
    Wildfire forecasting & {\color{ForestGreen}\textbf{80.5 $\pm$ 0.59}} & 70.1 $\pm$ 0.43 & 74.9 $\pm$ 0.16 & 71.0 $\pm$ 0.27 & 79.1 $\pm$ 0.47 & \textbf{75.8 $\pm$ 0.31} & \textit{AUPRC} \\
    \midrule
    \multicolumn{8}{l}{\textit{\textbf{Geophysical \& Socio-economic}}} \\
    \addlinespace[1pt]
    Air temperature & {\color{ForestGreen}\textbf{98.2 $\pm$ 0.02}} & 42.8 $\pm$ 2.15 & 90.6 $\pm$ 0.17 & \textbf{81.0 $\pm$ 2.83} & 93.9 $\pm$ 0.28 & 80.1 $\pm$ 0.85 & \textit{$R^2$} \\
    Median income & 28.1 $\pm$ 0.69 & 11.0 $\pm$ 0.54 & 32.9 $\pm$ 1.10 & 26.5 $\pm$ 1.32 & \textbf{50.2 $\pm$ 0.64} & \textbf{35.2 $\pm$ 1.06} & \textit{$R^2$} \\
    Cali.\ housing & 54.9 $\pm$ 0.59 & 31.1 $\pm$ 0.14 & 43.3 $\pm$ 0.13 & 39.2 $\pm$ 0.35 & 63.3 $\pm$ 0.36 & 35.3 $\pm$ 0.21 & \textit{$R^2$} \\
    Elevation & {\color{ForestGreen}\textbf{97.9 $\pm$ 0.03}} & \textbf{87.4 $\pm$ 0.05} & \textbf{94.1 $\pm$ 0.04} & {\color{RoyalBlue}\textbf{90.6 $\pm$ 0.18}} & 89.9 $\pm$ 0.10 & 64.9 $\pm$ 0.09 & \textit{$R^2$} \\
    Population & 80.2 $\pm$ 0.08 & \textbf{69.6 $\pm$ 0.06} & 77.0 $\pm$ 0.21 & 67.9 $\pm$ 0.13 & {\color{ForestGreen}\textbf{82.1 $\pm$ 0.14}} & 67.9 $\pm$ 0.09 & \textit{$R^2$} \\
    \midrule
    \multicolumn{8}{l}{\textit{\textbf{Biodiversity \& Ecology}}} \\
    \addlinespace[1pt]
    SatBird (USA summer) & {\color{ForestGreen}\textbf{67.5 $\pm$ 0.05}} & 56.5 $\pm$ 0.04 & 59.6 $\pm$ 0.03 & 57.9 $\pm$ 0.05 & 60.7 $\pm$ 0.07 & \textbf{58.2 $\pm$ 0.06} & \textit{top-$k$ Acc.} \\
    Countries & N/A & N/A & 80.9 $\pm$ 1.08 & 88.3 $\pm$ 0.17 & \textbf{95.5 $\pm$ 0.08} & \textbf{96.5 $\pm$ 0.07} & \textit{F1$_{w}$} \\
    iNaturalist & 49.1 $\pm$ 0.96 & 59.4 $\pm$ 0.20 & 40.4 $\pm$ 1.96 & 61.8 $\pm$ 0.49 & {\color{RoyalBlue}\textbf{56.4 $\pm$ 0.34}} & \textbf{69.8 $\pm$ 0.24} & \textit{Acc.} \\
    Biome & N/A & N/A & 84.0 $\pm$ 0.42 & 82.4 $\pm$ 0.15 & \textbf{89.8 $\pm$ 0.09} & \textbf{86.6 $\pm$ 0.10} & \textit{F1$_{w}$} \\
    Ecoregions & N/A & N/A & 59.9 $\pm$ 1.14 & 75.1 $\pm$ 0.06 & \textbf{79.7 $\pm$ 0.16} & \textbf{84.4 $\pm$ 0.14} & \textit{F1$_{w}$} \\
    \midrule
    \multicolumn{8}{l}{\textit{\textbf{Environmental}}} \\
    \addlinespace[1pt]
    Temperature & \multicolumn{2}{c}{N/A} & \multicolumn{2}{c}{{\color{ForestGreen}\textbf{79.8 $\pm$ 0.56}}} & \multicolumn{2}{c}{59.6 $\pm$ 2.20} & \textit{Ranking acc.} \\
    Precipitation & \multicolumn{2}{c}{N/A} & \multicolumn{2}{c}{{\color{ForestGreen}\textbf{77.7 $\pm$ 0.33}}} & \multicolumn{2}{c}{70.6 $\pm$ 2.60} & \textit{Ranking acc.} \\
    MSLP & \multicolumn{2}{c}{N/A} & \multicolumn{2}{c}{73.2 $\pm$ 0.60} & \multicolumn{2}{c}{{\color{ForestGreen}\textbf{83.2 $\pm$ 1.57}}} & \textit{Ranking acc.} \\
    Wind speed & \multicolumn{2}{c}{N/A} & \multicolumn{2}{c}{{\color{ForestGreen}\textbf{75.6 $\pm$ 0.46}}} & \multicolumn{2}{c}{71.6 $\pm$ 2.18} & \textit{Ranking acc.} \\
    Relative humidity & \multicolumn{2}{c}{N/A} & \multicolumn{2}{c}{\textbf{53.5 $\pm$ 0.92}} & \multicolumn{2}{c}{{\color{ForestGreen}\textbf{65.2 $\pm$ 1.62}}} & \textit{Ranking acc.} \\
    \bottomrule
  \end{tabular}
  }

  \vspace{1em}

  \resizebox{\textwidth}{!}{
  \begin{tabular}{lccccccc}
    \toprule
    \textbf{Task $\downarrow$ Model $\rightarrow$} & \multicolumn{2}{c}{\textbf{GeoCLIP}} & \multicolumn{2}{c}{\textbf{GT-Loc}} & \multicolumn{2}{c}{\begin{tabular}{@{}c@{}}\textbf{LoCA}\\[-0.3em]{\scriptsize (Ours)}\end{tabular}} & \multirow{2}{*}{\textbf{Metric}} \\
    \cmidrule(lr){2-3} \cmidrule(lr){4-5} \cmidrule(lr){6-7}
      & MLP & Linear & MLP & Linear & MLP & Linear & \\
    \midrule
    \multicolumn{8}{l}{\textit{\textbf{Land cover}}} \\
    \addlinespace[1pt]
    reBEN & 56.0 $\pm$ 0.18 & 44.1 $\pm$ 0.05 & 56.4 $\pm$ 0.35 & 52.4 $\pm$ 0.04 & {\color{ForestGreen}\textbf{58.3 $\pm$ 0.35}} & 54.7 $\pm$ 0.22 & \textit{F1$_{\mu}$} \\
    \midrule
    \multicolumn{8}{l}{\textit{\textbf{Wildfire forecasting}}} \\
    \addlinespace[1pt]
    Wildfire forecasting & 79.8 $\pm$ 0.25 & 73.2 $\pm$ 0.23 & 79.6 $\pm$ 0.60 & 75.5 $\pm$ 0.05 & \textbf{80.2 $\pm$ 0.43} & {\color{RoyalBlue}\textbf{76.2 $\pm$ 0.08}} & \textit{AUPRC} \\
    \midrule
    \multicolumn{8}{l}{\textit{\textbf{Geophysical \& Socio-economic}}} \\
    \addlinespace[1pt]
    Air temperature & \textbf{94.3 $\pm$ 0.07} & 32.9 $\pm$ 0.20 & 94.2 $\pm$ 0.13 & 60.4 $\pm$ 0.41 & 90.3 $\pm$ 0.44 & {\color{RoyalBlue}\textbf{81.4 $\pm$ 0.68}} & \textit{$R^2$} \\
    Median income & 45.0 $\pm$ 0.44 & 12.1 $\pm$ 0.12 & 45.9 $\pm$ 0.60 & 31.9 $\pm$ 0.13 & {\color{ForestGreen}\textbf{51.9 $\pm$ 0.50}} & {\color{RoyalBlue}\textbf{39.4 $\pm$ 0.67}} & \textit{$R^2$} \\
    Cali.\ housing & \textbf{77.3 $\pm$ 0.36} & \textbf{48.4 $\pm$ 0.11} & {\color{ForestGreen}\textbf{78.3 $\pm$ 0.21}} & {\color{RoyalBlue}\textbf{66.1 $\pm$ 0.05}} & 65.5 $\pm$ 0.65 & 36.4 $\pm$ 0.12 & \textit{$R^2$} \\
    Elevation & 83.0 $\pm$ 0.27 & 56.0 $\pm$ 0.01 & 83.9 $\pm$ 0.18 & 62.5 $\pm$ 0.01 & 89.4 $\pm$ 0.11 & 64.4 $\pm$ 0.01 & \textit{$R^2$} \\
    Population & 78.1 $\pm$ 0.15 & 66.0 $\pm$ 0.02 & 78.5 $\pm$ 0.18 & {\color{RoyalBlue}\textbf{69.7 $\pm$ 0.01}} & \textbf{81.8 $\pm$ 0.11} & 68.2 $\pm$ 0.02 & \textit{$R^2$} \\
    \midrule
    \multicolumn{8}{l}{\textit{\textbf{Biodiversity \& Ecology}}} \\
    \addlinespace[1pt]
    SatBird (USA summer) & \textbf{63.3 $\pm$ 0.06} & 47.8 $\pm$ 0.11 & \textbf{63.3 $\pm$ 0.01} & 52.7 $\pm$ 0.01 & 60.8 $\pm$ 0.02 & {\color{RoyalBlue}\textbf{61.1 $\pm$ 0.06}} & \textit{top-$k$ Acc.} \\
    Countries & 91.5 $\pm$ 0.20 & 79.8 $\pm$ 0.02 & 94.7 $\pm$ 0.06 & 94.4 $\pm$ 0.02 & {\color{RoyalBlue}\textbf{96.4 $\pm$ 0.03}} & {\color{ForestGreen}\textbf{97.0 $\pm$ 0.02}} & \textit{F1$_{w}$} \\
    iNaturalist & 44.8 $\pm$ 1.33 & 58.4 $\pm$ 0.24 & 49.3 $\pm$ 1.23 & 59.3 $\pm$ 0.25 & \textbf{55.0 $\pm$ 0.34} & {\color{ForestGreen}\textbf{71.4 $\pm$ 0.19}} & \textit{Acc.} \\
    Biome & 86.9 $\pm$ 0.21 & 75.6 $\pm$ 0.03 & 88.1 $\pm$ 0.11 & 81.9 $\pm$ 0.02 & {\color{ForestGreen}\textbf{91.0 $\pm$ 0.07}} & {\color{RoyalBlue}\textbf{87.0 $\pm$ 0.04}} & \textit{F1$_{w}$} \\
    Ecoregions & 60.6 $\pm$ 0.66 & 42.6 $\pm$ 0.03 & 69.6 $\pm$ 0.30 & 75.9 $\pm$ 0.03 & {\color{RoyalBlue}\textbf{80.2 $\pm$ 0.13}} & {\color{ForestGreen}\textbf{85.0 $\pm$ 0.13}} & \textit{F1$_{w}$} \\
    \midrule
    \multicolumn{8}{l}{\textit{\textbf{Environmental}}} \\
    \addlinespace[1pt]
    Temperature & \multicolumn{2}{c}{72.0 $\pm$ 0.73} & \multicolumn{2}{c}{74.4 $\pm$ 0.92} & \multicolumn{2}{c}{\textbf{77.9 $\pm$ 0.90}} & \textit{Ranking acc.} \\
    Precipitation & \multicolumn{2}{c}{72.1 $\pm$ 0.49} & \multicolumn{2}{c}{72.5 $\pm$ 0.63} & \multicolumn{2}{c}{{\color{ForestGreen}\textbf{77.7 $\pm$ 1.02}}} & \textit{Ranking acc.} \\
    MSLP & \multicolumn{2}{c}{77.8 $\pm$ 0.71} & \multicolumn{2}{c}{\textbf{79.1 $\pm$ 0.69}} & \multicolumn{2}{c}{76.3 $\pm$ 1.07} & \textit{Ranking acc.} \\
    Wind speed & \multicolumn{2}{c}{73.8 $\pm$ 0.72} & \multicolumn{2}{c}{71.4 $\pm$ 1.63} & \multicolumn{2}{c}{{\color{ForestGreen}\textbf{75.6 $\pm$ 1.49}}} & \textit{Ranking acc.} \\
    Relative humidity & \multicolumn{2}{c}{47.8 $\pm$ 0.85} & \multicolumn{2}{c}{43.5 $\pm$ 1.38} & \multicolumn{2}{c}{47.5 $\pm$ 1.41} & \textit{Ranking acc.} \\
    \bottomrule
  \end{tabular}
  }
\end{table*}

\subsubsection{Comparison among state-of-the-art encoders}

\Cref{tab:head-to-head} gives a pairwise view of the results in \Cref{tab:benchmark-results-main}. For each model and task we use its best probe (the LightGBM score for climate ranking), and each pair is compared only on the tasks both models were evaluated on: 9 for pairs involving AlphaEarth, 17 otherwise. Ties at the displayed precision count for neither model, and the win rate pools a model's wins over all of its pairwise comparisons.

LoCA is the only encoder that beats every other model on a majority of shared tasks, and it has the highest win rate (68.8\%). Its advantage is therefore broad. It is never outperformed on land-cover or ecological classification, and its losses concentrate on geophysical and climate targets. Its most frequent loss, air temperature, comes from the MLP probe: under the linear probe LoCA obtains the best air-temperature score of all models (\Cref{tab:benchmark-results-main}). This is consistent with its high effective rank (\Cref{tab:effrank}), since location information is linearly accessible in LoCA's embeddings, whereas other encoders need a non-linear probe to reach it. The most controlled comparison is with SatCLIP, which shares LoCA's location-encoder architecture: LoCA wins 12 of these 17 tasks, which isolates the benefit of its co-aligned, high-rank multimodal target.

Beyond LoCA, the ranking follows how each encoder's target is formed. SatCLIP, which like LoCA aligns to EO imagery, is second by majority wins (4/5). The encoders anchored to models trained on natural images rank lower: GT-Loc and GeoCLIP win majorities only against each other, and Copernicus-FM wins no majority despite its 768-dimensional, imagery-derived embeddings. This supports the argument of \Cref{sec:regularization} that rank obtained from an external anchor partly reflects axes of variation that are not specific to EO, and does not translate into downstream performance. AlphaEarth is LoCA's closest competitor, since LoCA wins 5 of 9 tasks.

\begin{table}[t]
  \caption{Head-to-head comparison over the 17 tasks of \Cref{tab:benchmark-results-main}
  (best probe per model). Cell $(i,j)$: tasks on which the row model outperforms the column model, out of those both were evaluated on;
  {\color{ForestGreen}\textbf{green}} marks a majority. Ties at displayed precision count for neither model. LoCA is the only encoder that beats every other encoder on a majority of shared tasks.}
  \label{tab:head-to-head}
  \centering
  \resizebox{\columnwidth}{!}{%
  \begin{tabular}{lcccccc|cc}
    \toprule
    Row beats column $\rightarrow$ & LoCA & SatCLIP & AlphaEarth & GT-Loc & GeoCLIP & Copernicus FM & \begin{tabular}{@{}c@{}}Majority\\wins\end{tabular} & \begin{tabular}{@{}c@{}}Win\\rate\end{tabular} \\
    \midrule
    \textbf{LoCA (Ours)} & -- & {\color{ForestGreen}\textbf{12/17}} & {\color{ForestGreen}\textbf{5/9}} & {\color{ForestGreen}\textbf{13/17}} & {\color{ForestGreen}\textbf{12/17}} & {\color{ForestGreen}\textbf{11/17}} & {\color{ForestGreen}\textbf{5/5}} & 68.8\% \\
    SatCLIP & 5/17 & -- & {\color{ForestGreen}\textbf{5/9}} & {\color{ForestGreen}\textbf{11/17}} & {\color{ForestGreen}\textbf{10/17}} & {\color{ForestGreen}\textbf{12/17}} & 4/5 & 55.8\% \\
    AlphaEarth & 4/9 & 4/9 & -- & {\color{ForestGreen}\textbf{6/9}} & {\color{ForestGreen}\textbf{6/9}} & {\color{ForestGreen}\textbf{6/9}} & 3/5 & 57.8\% \\
    GT-Loc & 4/17 & 6/17 & 3/9 & -- & {\color{ForestGreen}\textbf{12/17}} & {\color{ForestGreen}\textbf{10/17}} & 2/5 & 45.5\% \\
    GeoCLIP & 5/17 & 7/17 & 3/9 & 4/17 & -- & {\color{ForestGreen}\textbf{9/17}} & 1/5 & 36.4\% \\
    Copernicus FM & 4/17 & 5/17 & 3/9 & 7/17 & 8/17 & -- & 0/5 & 35.1\% \\
    \bottomrule
  \end{tabular}}
\end{table}

\subsubsection{Supervised co-aligned targets}\label{app:supervised}
\Cref{tab:reben-resnet18-perclass} provides the per-class results for the supervised BigEarthNet experiment summarized in the main paper. Rows marked ``(ours)'' use the same frozen base encoders as the corresponding baseline row, but train the adapters and shared projection head with MV-DHEL.

\begin{table*}[t]
  \caption{Per-class F1 on BigEarthNet land-cover classification (10 classes),
  comparing pure location encoders (rows 1--2) against supervised ResNet-18
  image encoders with and without the MV-DHEL multimodal training.
  ``Mean F1'' is the unweighted mean across the 10 class-wise F1 scores;
  remaining columns report class-wise F1.
  \textbf{Best per column} is highlighted in {\color{ForestGreen}\textbf{green}}
  and \textbf{second best} in \textbf{bold}.}
  \label{tab:reben-resnet18-perclass}
  \centering
  \resizebox{\textwidth}{!}{%
  \begin{tabular}{lcccccccccccc}
    \toprule
    \textbf{Model}
      & \textbf{Mean F1}
      & \begin{tabular}[c]{@{}c@{}}\textbf{Arable}\\\textbf{Land}\end{tabular}
      & \begin{tabular}[c]{@{}c@{}}\textbf{Mixed}\\\textbf{Forest}\end{tabular}
      & \begin{tabular}[c]{@{}c@{}}\textbf{Coniferous}\\\textbf{Forest}\end{tabular}
      & \begin{tabular}[c]{@{}c@{}}\textbf{Transitional}\\\textbf{Woodland}\end{tabular}
      & \begin{tabular}[c]{@{}c@{}}\textbf{Broad-leaved}\\\textbf{Forest}\end{tabular}
      & \begin{tabular}[c]{@{}c@{}}\textbf{Urban}\\\textbf{Fabric}\end{tabular}
      & \begin{tabular}[c]{@{}c@{}}\textbf{Inland}\\\textbf{Waters}\end{tabular}
      & \begin{tabular}[c]{@{}c@{}}\textbf{Complex}\\\textbf{Cultivation}\end{tabular}
      & \begin{tabular}[c]{@{}c@{}}\textbf{Marine}\\\textbf{Waters}\end{tabular}
      & \textbf{Pastures} \\
    \midrule

    Location
      & 74.7 & 71.0 & 77.1 & 80.4 & 67.5 & 73.8 & 65.8 & 72.8 & 69.6 & 89.3 & 79.6 \\
    Location (ours)
      & 76.6 & 71.4 & 79.8 & 83.8 & 68.5 & 77.6 & 63.0 & 73.0 & 70.7 & \textbf{97.8} & 80.6 \\

    \midrule

    S1
      & 69.6 & 72.9 & 70.0 & 73.1 & 63.8 & 64.0 & 75.9 & 78.4 & 52.5 & 95.2 & 50.1 \\
    S1 (ours)
      & 75.9 & 76.1 & 74.6 & 76.7 & 67.0 & 70.2 & \textbf{78.0} & 84.7 & 64.8 & 95.6 & 70.9 \\

    \midrule

    S2
      & 78.4 & 78.5 & 78.4 & 84.5 & 68.9 & 76.7 & 73.8 & 82.5 & 67.9 & 96.1 & 76.6 \\
    S2 (ours)
      & \textbf{80.9} & \textbf{81.7} & \textbf{81.7} & \textbf{86.9} & \textbf{71.4} & 79.1 & 75.1 & 83.5 & 71.4 & 96.7 & \textbf{81.2} \\

    \midrule

    S1+S2
      & 78.1 & 78.0 & 78.4 & 83.7 & 67.8 & 76.6 & 72.2 & 83.5 & 67.9 & 97.2 & 76.1 \\
    S1+S2 (ours)
      & 80.8 & \textbf{81.7} & \textbf{81.7} & 86.4 & 70.9 & 78.9 & 74.3 & 83.6 & \textbf{72.0} & 97.3 & \textbf{81.2} \\

    \midrule

    All
      & \textbf{80.9} & 81.0 & 80.9 & 85.8 & 70.6 & \textbf{79.2} & 76.9 & \textbf{85.5} & 71.3 & \textbf{97.8} & 79.8 \\
    All (ours)
      & {\color{ForestGreen}\textbf{84.3}} & {\color{ForestGreen}\textbf{83.9}}
      & {\color{ForestGreen}\textbf{83.6}} & {\color{ForestGreen}\textbf{88.7}}
      & {\color{ForestGreen}\textbf{73.8}} & {\color{ForestGreen}\textbf{83.5}}
      & {\color{ForestGreen}\textbf{80.3}} & {\color{ForestGreen}\textbf{89.0}}
      & {\color{ForestGreen}\textbf{77.0}} & {\color{ForestGreen}\textbf{99.3}}
      & {\color{ForestGreen}\textbf{83.9}} \\

    \bottomrule
  \end{tabular}%
  }
\end{table*}

The MV-DHEL adapter stage improves the mean F1 for every modality family. The location encoder improves from $74.7$ to $76.6$, S1 from $69.6$ to $75.9$, S2 from $78.4$ to $80.9$, S1+S2 from $78.1$ to $80.8$, and the full four-modality system from $80.9$ to $84.3$. The largest gain is for S1, the weakest standalone image modality, suggesting that co-aligned multimodal supervision is especially helpful when one view is individually less discriminative. The full system, All (ours), obtains the best score in every class, supporting the claim that Stage~2 can exploit an already co-aligned supervised target even without Stage~1.

\subsubsection{Ablating LoCA's Multi-View Components.}
The purpose of this ablation is to separate the two distinct roles played by
LoCA's additional views: joint Stage~1 pretraining should construct a coherent
Sentinel-1/Sentinel-2 target for Stage~2, while the metadata encoder should add
complementary non-image signal only on tasks where such signal is useful.
Table~\ref{tab:app_ablation} reports the complete 17-task version of the
ablation.

The main failure mode is the removal of joint Stage~1 pretraining: when the
Sentinel-1 and Sentinel-2 targets are replaced by independently pretrained MoCo
encoders, Stage~2 no longer receives a co-aligned multi-view target, and
performance collapses on nearly all regression tasks and is substantially worse
on classification tasks.  For example, reBEN
drops from $54.7$ to $1.1$ micro-F1, air-temperature prediction from $81.4$ to
$-5.8$ $R^2$, and Countries from $97.0$ to $4.4$ weighted F1. The effect is particularly clear on the climate-ranking tasks: pairwise
accuracy lies between $49.6\%$ and $50.6\%$ across all five variables, i.e., approximately chance-level ordering. This supports the role of
\textit{Stage~1} in constructing a coherent multimodal target before adapting the location encoder and re-enforces the fact that SatCLIP cannot be trivially extended to multiple modalities

Removing metadata gives a more selective effect, which is the expected behavior. The metadata view is not expected to help uniformly
across all tasks: when Sentinel-1, Sentinel-2, or coarse geography already
capture the relevant signal, adding metadata should have little effect. Metadata improves the full model over the
metadata-free variant on 9 of the 17 tasks. The largest gains are observed for several geophysical and climate-ranking tasks, including air
temperature ($76.0$ to $81.4$), elevation ($57.1$ to $64.4$), MSLP ranking ($70.0\%$ to $76.3\%$), and wind-speed ranking ($71.2\%$ to $75.6\%$). 

Overall, the ablation indicates two qualitatively different effects. \textit{Stage~1} is necessary for obtaining a useful multimodal location
representation, with its removal causing broad performance collapse and chance-level climate ranking. The metadata view instead contributes
task-dependent complementary information: it improves several targets, but does not provide a uniform gain across all downstream tasks.

\begin{table}[t]
  \caption{Ablation isolating the contributions of (i) the joint \textit{Stage-1} contrastive pretraining and
  (ii) the metadata encoder, under the linear probe. Values are metric $\times\,100$; higher is better.
  Highlighting follows \Cref{tab:benchmark-results-main}. Without \textit{Stage 1} the embeddings collapse:
  downstream probes drop to (or below) the trivial predictor and environmental ranking is at chance ($50$).
  All variants are evaluated on all 17 tasks.}
  \label{tab:app_ablation}
  \centering
  \begin{tabular}{lcccc}
    \toprule
    \textbf{Task $\downarrow$ Variant $\rightarrow$}
      & \begin{tabular}{@{}c@{}}\textbf{LoCA}\\[-0.3em]{\scriptsize w/o \textit{Stage 1}}\end{tabular}
      & \begin{tabular}{@{}c@{}}\textbf{LoCA}\\[-0.3em]{\scriptsize w/o metadata}\end{tabular}
      & \begin{tabular}{@{}c@{}}\textbf{LoCA}\\[-0.3em]{\scriptsize (full)}\end{tabular}
      & \textbf{Metric} \\
    \midrule
    \multicolumn{5}{l}{\textit{\textbf{Biodiversity \& Ecology}}} \\
    \addlinespace[1pt]
    Countries & 4.4 & \textbf{96.9} & {\color{ForestGreen}\textbf{97.0}} & \textit{F1$_{w}$} \\
    Biome & 9.7 & \textbf{85.8} & {\color{ForestGreen}\textbf{87.0}} & \textit{F1$_{w}$} \\
    Ecoregions & 1.7 & {\color{ForestGreen}\textbf{85.4}} & \textbf{85.0} & \textit{F1$_{w}$} \\
    iNaturalist & 56.6 & {\color{ForestGreen}\textbf{71.6}} & \textbf{71.4} & \textit{Acc.} \\
    SatBird (USA summer) & 44.5 & {\color{ForestGreen}\textbf{61.2}} & \textbf{61.1} & \textit{top-$k$ Acc.} \\
    \midrule
    \multicolumn{5}{l}{\textit{\textbf{Land cover}}} \\
    \addlinespace[1pt]
    reBEN & 1.1 & {\color{ForestGreen}\textbf{55.2}} & \textbf{54.7} & \textit{F1$_{\mu}$} \\
    \midrule
    \multicolumn{5}{l}{\textit{\textbf{Wildfire forecasting}}} \\
    \addlinespace[1pt]
    Wildfire forecasting & 49.5 & {\color{ForestGreen}\textbf{76.5}} & \textbf{76.2} & \textit{AUPRC} \\
    \midrule
    \multicolumn{5}{l}{\textit{\textbf{Geophysical \& Socio-economic}}} \\
    \addlinespace[1pt]
    Median income & -7.9 & \textbf{36.5} & {\color{ForestGreen}\textbf{39.4}} & \textit{$R^2$} \\
    Population & -0.8 & \textbf{66.0} & {\color{ForestGreen}\textbf{68.2}} & \textit{$R^2$} \\
    Cali.\ housing & -1.5 & {\color{ForestGreen}\textbf{36.8}} & \textbf{36.4} & \textit{$R^2$} \\
    Air temperature & -5.8 & \textbf{76.0} & {\color{ForestGreen}\textbf{81.4}} & \textit{$R^2$} \\
    Elevation & -0.7 & \textbf{57.1} & {\color{ForestGreen}\textbf{64.4}} & \textit{$R^2$} \\
    \midrule
    \rowcolor{orange!12}
    \textit{\textbf{Environmental}} & \multicolumn{3}{c}{\textbf{LightGBM ranking}} & \\
    \cmidrule(lr){2-4}
    \addlinespace[1pt]
    Temperature & 49.9 & {\color{ForestGreen}\textbf{84.4}} & \textbf{77.9} & \textit{Ranking acc.} \\
    Precipitation & 49.6 & \textbf{75.9} & {\color{ForestGreen}\textbf{77.7}} & \textit{Ranking acc.} \\
    MSLP & 49.9 & \textbf{70.0} & {\color{ForestGreen}\textbf{76.3}} & \textit{Ranking acc.} \\
    Wind speed & 50.6 & \textbf{71.2} & {\color{ForestGreen}\textbf{75.6}} & \textit{Ranking acc.} \\
    Relative humidity & {\color{ForestGreen}\textbf{49.8}} & 47.1 & \textbf{47.5} & \textit{Ranking acc.} \\
    \midrule
    \textbf{\# Best (out of 17)} & 1 & 7 & {\color{ForestGreen}\textbf{9}} & \\
    \bottomrule
  \end{tabular}
\end{table}

\section{Extended Related Work}

\subsection{Geospatial Foundation Models (GFMs) \& Embeddings.}
Several GFMs have relied on Masked Image Modelling (MIM) objectives as their pretraining strategy. ScaleMAE~\cite{reed2023scale} and Cross-Scale MAE~\cite{tang2023cross} utilize MIM with GSD-aware positional encoding and cross-prediction for multi-scale learning. USatMAE~\cite{irvin2023usat} and SpectralMAE~\cite{zhu2023spectralmae} introduce spectral-aware masking for paired NAIP/Sentinel-2 and hyperspectral data respectively. Extending the MIM paradigm, Presto~\cite{tseng2023lightweight}, FoMo-Net~\cite{bountos2025fomo}, DOFA~\cite{xiong2024neural},  Copernicus-FM~\cite{wang2025towards}, OlmoEarth~\cite{herzog2025olmoearth}, THOR~\cite{forgaard2026thor}  modify the MIM framework to train sensor agnostic encoders.  
Finally, Galileo~\cite{tseng2025galileo} blends MIM with contrastive objectives introducing a dual global/local contrastive loss for multi-scale feature learning.
Following a different direction, SkySense~\cite{guo2024skysense} aligns WorldView RGB, Sentinel-2, and Sentinel-1 through contrastive and cross-modal objectives. In a similar fashion, CROMA~\cite{fuller2023croma} and OmniSAT~\cite{astruc2024omnisat} fuse Sentinel-1/2 via contrastive and reconstruction losses, with CROMA limited to single timestamps and OmniSAT extending to time series and aerial data.
Recent works have explored generative approaches for the creation of GFMs
TerraMind~\cite{jakubik2025terramind} is the first any-to-any generative multimodal EO model, pretrained autoregressively on dual-scale token- and pixel-level representations across nine modalities. DiffusionSat~\cite{khanna2023diffusionsat} conditions on geolocation metadata for temporal generation, super-resolution, and in-painting. SatDiFuser~\cite{jia2025can} demonstrates the potential of diffusion-based approaches for the creation of GFMs.  Driven by the increasing capabilities of GFMs, recent works have utilized them to store precomputed location embeddings, promising to replace raw observations altogether. AlphaEarth~\cite{brown2025alphaearth} fuses optical, radar, LiDAR, and climate reanalysis data into 64-dimensional embeddings at 10m resolution, released annually. Copernicus-FM~\cite{wang2025towards} provides a global embedding map at $0.25^{\circ}$ by averaging multi-sensor representations. TESSERA~\cite{feng2026tessera} applies a Barlow Twins pretraining framework, releasing annual embeddings with \textit{int8} quantization. However, the pretraining schemes underlying these models are observation-based rather than location-based, yielding degenerate representations that undermine their stated purpose. 

\subsection{Location Encoders.} Location encoders follow a different paradigm, aiming to learn location-specific representations beyond single observations, identifying the key characteristics and conditions that govern and connect geospatial locations. In \cite{MacAodha2019PresenceOnlyGP}  the authors introduced an early location encoder training a fully connected residual network (FCNet) to estimate the probability of object categories at a given location and time. Sphere2Vec~\cite{mai2023sphere2vec} proposed a multi-scale encoding scheme based on Double Fourier Sphere (DFS) features that preserves geodesic distances on the spherical surface. In SirenNet~\cite{russwurm2023geographic} the authors introduced spherical harmonic basis functions as positional encoding which they utilized in a sinusoidal representation networks. Contrastive learning has emerged as the dominant paradigm for training location encoders at scale. GeoCLIP~\cite{cepeda2023geoclip} aligns CLIP-based image embeddings with GPS coordinates. SatCLIP~\cite{klemmer2025satclip} extends the contrastive paradigm to satellite imagery, training a location encoder against Sentinel-2, frozen image encoders to produce general-purpose location embeddings. Recent approaches have moved beyond unimodal location encoders. GAIR~\cite{liu2025gair} jointly aligns remote sensing imagery, street-view photographs, and geolocation through two contrastive losses: one aligning RS with street-view embeddings and another anchoring each visual modality to the location representation. This pairwise formulation, however, does not scale gracefully: each additional modality requires new loss terms that increasingly distort the shared embedding space. GT-Loc~\cite{shatwell2025gt} incorporates temporal context by mapping images, locations, and timestamps into a joint embedding space, pairing a frozen pretrained image encoder with learned location and time encoders optimized via a GeoCLIP-style loss for image--location alignment and a cross-entropy loss for image--time alignment. While effective for three modalities, the heterogeneous loss design again constrains extensibility. TIGeR~\cite{shatwell2026tiger} addresses pairwise completeness by aligning image--location--time triplets through six InfoNCE losses covering all pairwise combinations. This is suboptimal for two reasons: a)~the competing objectives produce conflicting gradients that hinder the optimization process, and b) the number of required losses grows quadratically with the number of modalities, making further extension impractical. We address these limitations by generalizing the SatCLIP framework to arbitrarily many modalities through a single multi-view contrastive objective.

\subsection{Multi-View Contrastive Learning.} The effectiveness of leveraging multiple views in CL has been demonstrated by aggregating pairwise losses from views generated through various means, such as multiple projection heads \cite{wang2024adaptive}, image scales \cite{li2023global}, classes \cite{shah2023multi}, and feature levels \cite{xu2022multi}. However, the way to efficiently utilize multiple views in Contrastive Learning has been indipendently studied. 
 An early approach by \cite{tian2020contrastive} framed multi-view learning as a collection of independent tasks, treating each combination of views separately and simply aggregating pairwise losses. Recently, \cite{shidani2024polyview} has incorporated another aggregation method, while utilizing optimal transport on a high similarity tensor was proposed by \cite{piran2024contrasting}. The MV-DHEL \citep{koromilas2025principled} objective was recently proposed and is able to properly utilize all view interactions resulting to consistent performance gains that scale with the number of views and embeddings that avoid dimensionality collapse. In this work we draw from the MV-DHEL benefits to train multimodal location encoders.

\section{Licenses and terms for existing assets}
\label{app:asset_licenses}

We use existing datasets, models, and code only for research purposes and credit the original creators through the citations in the main paper and appendix. \Cref{tab:asset-licenses} summarizes the licenses and terms of use for the external assets used for pretraining, metadata construction, baselines, and downstream evaluation. We do not redistribute third-party raw imagery, pretrained baseline checkpoints, or third-party benchmark datasets unless their license explicitly permits redistribution.

\begin{table*}[p]
\centering
\footnotesize
\caption{Existing assets used in this work and their licenses or terms of use. ``Use in this paper'' describes how the asset is used; it does not imply that we redistribute the original asset.}
\label{tab:asset-licenses}
\renewcommand{\arraystretch}{1.20}
\setlength{\tabcolsep}{5pt}
\begin{tabular}{
>{\raggedright\arraybackslash}p{0.23\textwidth}
>{\raggedright\arraybackslash}p{0.32\textwidth}
>{\raggedright\arraybackslash}p{0.36\textwidth}
}
\toprule
\textbf{Asset} & \textbf{Use in this paper} & \textbf{License / terms}  \\
\midrule

SatCLIP S2-100K dataset
& Starting point for coordinate sampling and Sentinel-2 pairing.
& Research purposes only; license field not specified in the dataset card. \\
\midrule

SatCLIP code and pretrained models
& Baseline location encoder; architecture reference.
& MIT License. \\
\midrule

Copernicus DEM GLO-30
& Elevation, slope, aspect, and curvature metadata.
& Copernicus DEM free license. \\
\midrule

WorldPop population density
& Population-density metadata.
& Creative Commons Attribution 4.0 International (CC BY 4.0). \\
\midrule

WorldPop distance-to-roads product
& Road-distance metadata.
& Creative Commons Attribution 4.0 International (CC BY 4.0). \\
\midrule

OpenStreetMap
& Underlying source for road-distance data.
& Open Data Commons Open Database License 1.0 (ODbL 1.0). \\
\midrule

C3S / ESA CCI Land Cover
& Land-cover fraction metadata.
& Dataset page lists ESA CCI licence, CC-BY licence, and VITO licence. \\
\midrule

AlphaEarth Foundations satellite embedding dataset
& Observation-based embedding-field baseline.
& Creative Commons Attribution 4.0 International (CC BY 4.0). \\
\midrule

Copernicus-FM and Copernicus-Embed-025deg
& Observation-based embedding-field baseline.
& Code: Apache-2.0 with third-party MIT/CC-BY-NC-4.0 components; pretrained weights and datasets: CC BY 4.0; Copernicus-Embed-025deg: CC BY 4.0. \\
\midrule

GeoCLIP
& Baseline location encoder.
& MIT License. \\
\midrule

GT-Loc
& Baseline location encoder.
& Creative Commons Attribution-NonCommercial 4.0 International (CC BY-NC 4.0). \\
\midrule

SatBird
& Biodiversity downstream benchmark.
& Repository license: GPL-3.0. \\
\midrule

BigEarthNet v2.0 / reBEN
& Land-cover downstream benchmark; supervised encoder experiment.
& Community Data License Agreement -- Permissive, Version 1.0 (CDLA-Permissive-1.0). \\
\midrule

SeasFire Cube
& Long-term climate-ranking downstream tasks.
& Creative Commons Attribution 4.0 International (CC BY 4.0). \\
\midrule

Mesogeos
& Wildfire danger forecasting downstream task.
& Creative Commons Attribution 4.0 International (CC BY 4.0). \\
\midrule

iNaturalist / iNat competition data
& Species-prediction downstream task.
& iNaturalist competition terms: non-commercial research and educational use; no redistribution of images. \\
\midrule

RESOLVE Ecoregions 2017
& Countries/biome/ecoregion-style ecological labels and evaluation resources.
& Creative Commons Attribution 4.0 International (CC BY 4.0). \\
\midrule

California Housing via scikit-learn
& Downstream regression benchmark.
& scikit-learn is BSD-3-Clause; original dataset is from StatLib. \\
\midrule

USDA ERS median household income data
& Median-income downstream regression benchmark.
& U.S. Government public-domain information / USDA website terms. \\
\midrule

U.S. Census TIGER/Line county boundaries
& County-coordinate construction for median-income task.
& U.S. Census Bureau public geographic data. \\

\bottomrule
\end{tabular}
\end{table*}

\end{document}